\documentclass[10pt]{article}

\usepackage[allcolors=NavyBlue,colorlinks=true,backref=page]{hyperref}       % hyperlinks
\usepackage{url}
\usepackage[preprint,nohyperref]{jmlr2e}
\usepackage{algorithm,algorithmic}
\usepackage{amsmath, amssymb, amsfonts}
\usepackage{xspace}
\usepackage{dsfont}
\usepackage{adjustbox}
\usepackage{booktabs}       % professional-quality tables
\usepackage{float}
\usepackage{ftnxtra}

\usepackage{graphicx}
\usepackage{microtype}      % microtypography
\usepackage{nicefrac}       % compact symbols for 1/2,
\usepackage{setspace}

\usepackage[disable]{todonotes}
\usepackage{mathtools}
\usepackage{thmtools}
\usepackage{xspace}

\usepackage{wrapfig}
\usepackage{subcaption}

\usepackage{mathtools}
\usepackage{bm}

\usepackage{array}
\usepackage{mathrsfs}

\usepackage{bbding} % big plus marker \Plus
\def\RPlus{\ensuremath{\mathbin{\rule[.13em]{.66em}{.22em}\hspace{-.44em}\rule[-.08em]{.22em}{.66em}\,}}}

\DeclareRobustCommand\sampleline[1]{%
  \tikz\draw[#1] (0,0) (0,\the\dimexpr\fontdimen22\textfont2\relax)
  -- (2em,\the\dimexpr\fontdimen22\textfont2\relax);%
}

\newtheorem{rem}{Remark}
\declaretheorem[name=Corollary]{cor} % thmtools is needed
\declaretheorem[name=Lemma]{lem}
\declaretheorem[name=Proposition]{prop}

\definecolor{NavyBlue}{rgb}{0,0.08,0.50}
\definecolor{Maroon}{RGB}{110,76,75}

\definecolor{ForestGreen}{RGB}{74,103,65}
\definecolor{YellowOrange}{RGB}{255, 174, 66}

\definecolor{mydarkblue}{rgb}{0,0.08,0.45}
\definecolor{darkgreen}{RGB}{34,139,34}

\definecolor{myblue}{RGB}{49,130,189}
\definecolor{myred}{RGB}{251,106,74}

\usepackage{amsmath,amsfonts,bm}

\def\figref#1{figure~\ref{#1}}
\def\Figref#1{Figure~\ref{#1}}

\def\secref#1{section~\ref{#1}}
\def\eqref#1{equation~\ref{#1}}
\def\algref#1{algorithm~\ref{#1}}
\def\1{\bm{1}}

\DeclareMathAlphabet{\mathsfit}{\encodingdefault}{\sfdefault}{m}{sl}
\SetMathAlphabet{\mathsfit}{bold}{\encodingdefault}{\sfdefault}{bx}{n}

\DeclareMathOperator*{\argmax}{arg\,max}

\newcommand{\defeq}{\triangleq}
\newcommand{\tabref}[1]{table~\ref{#1}}

\newcommand{\corref}[1]{Corollary~\ref{#1}}
\newcommand{\propref}[1]{Proposition~\ref{#1}}
\newcommand{\lemref}[1]{Lemma~\ref{#1}}

\newcommand{\algoref}[1]{Algorithm~\ref{#1}}
\newcommand{\lineref}[1]{line~\ref{#1}}

\newcommand{\paren} [1] {\ensuremath{ \left( {#1} \right) }}

\newcommand{\bracket}[1]{\left[#1\right]}

\newcommand{\curlybracket}[1]{\ensuremath{\left\{#1\right\}}}

\newcommand{\reals}{\ensuremath{\mathbb{R}}}

\renewcommand{\Pr}[1]{\ensuremath{\mathbb{P}\left[#1\right] }}

\newcommand{\mutualinfo}[1]{\mathbb{I}\paren{#1}}

\newcommand{\denselist}{\itemsep 0pt\topsep-10pt\partopsep-6pt}

\newcommand{\by}{{\mathbf{y}}}

\usepackage{ifthen}

\newboolean{showcomments}
\setboolean{showcomments}{true}

\newcommand{\algname}{\textsf{Algname}\xspace}

\newcommand{\maxInfo}{\gamma}
\newcommand{\algParam}{\theta}

\definecolor{ultramarine}{RGB}{24,13,191}

\newif\iffinal
\finaltrue

\iffinal
    \newcommand{\fix}[1]{}
    \newcommand{\yuxin}[1]{}
    \newcommand{\fengxue}[1]{}
    \newcommand{\yuxinil}[1]{}
    \newcommand{\TBD}[1]{}
    \newcommand{\rebuttal}[1]{#1}

\else
    \newcommand{\fix}[1]{{\color{red} #1}}
    \newcommand{\yuxin}[1]{\todo[fancyline,color=Maroon!40]{YC: #1}\xspace}
    
    \newcommand{\fengxue}[1]{\todo[fancyline,color=ForestGreen!40]{FZ: #1}\xspace}
    \newcommand{\yuxinil}[1]{{\textcolor{Maroon}{[{\bf YC:} #1]}{}}}
    \newcommand{\rebuttal}[1]{{\color{purple} #1}}
    \newcommand{\TBD}[1]{{[{\bf TD:} \color{Maroon}#1]}}
\fi

\renewcommand{\algname}{\textsc{BALLET}\xspace}

\newcommand{\UCB}{\textsc{UCB}\xspace}

\newcommand{\SE}{\textsc{SE}\xspace}
\newcommand{\Linear}{\textsc{Linear}\xspace}

\newcommand{\instance}[0]{\ensuremath{\mathbf{x}}}

\newcommand{\GramMat}[0]{\ensuremath{\mathbf{K}}}
\newcommand{\Selected}[0]{\ensuremath{\mathbf{D}}}

\newcommand{\LatentRepSet}[0]{\ensuremath{\mathbf{Z}}}
\newcommand{\featureExtractor}[0]{q}
\newcommand{\searchSpace}[0]{\ensuremath{\mathbf{X}}}
\newcommand{\roi}[0]{{\ensuremath{\hat{\searchSpace}}}}
\newcommand{\LCB}[0]{{\ensuremath{\textrm{LCB}}}}
\newcommand{\CI}[0]{{\ensuremath{\textrm{CI}}}}
\newcommand{\UCBit}[0]{{\ensuremath{\textrm{UCB}}}}
\newcommand{\acq}[0]{{\ensuremath{\alpha_{\hat{f}}}}}
\newcommand{\acqUCB}[0]{{\ensuremath{\widehat{\textrm{UCB}}}}}
\newcommand{\acqLCB}[0]{{\ensuremath{\widehat{\textrm{LCB}}}}}
\newcommand{\acqUCBt}[0]{{\ensuremath{\widetilde{\textrm{UCB}}}}}
\newcommand{\acqLCBt}[0]{{\ensuremath{\widetilde{\textrm{LCB}}}}}
\newcommand{\acqCI}[0]{{\ensuremath{\widehat{\textrm{CI}} }}}
\newcommand{\discreteSet}[0]{\ensuremath{\Tilde{D}}}
\newcommand{\actionSet}[0]{\ensuremath{\textit{A}}}
\newcommand{\maxInfoBALLET}[0]{\ensuremath{\widehat{\maxInfo}}}
\newcommand{\discreteROI}[0]{\ensuremath{\discreteSet_\roi}}

\newcommand{\hatCI}[0]{\ensuremath{\widehat{\CI}}}
\newcommand{\tildeCI}[0]{\ensuremath{\widetilde{\CI}}}

\newcommand{\globalf}{\ensuremath{f_\text{g}}}
\newcommand{\roif}{\ensuremath{\hat{f}}}
\newcommand{\acqRoici}{{\ensuremath{\alpha_{\hat{f},\textsc{RCI}}}}}
\newcommand{\acqTs}{{\ensuremath{\alpha_{\hat{f},\textsc{TS}}}}}

\newcommand{\roiRangeAt}[1]{\ensuremath{\Delta_{\textsc{roi},#1}}}

\newcommand{\interCI}{\textsc{\algname-ICI}\xspace}
\newcommand{\interUCB}{\textsc{\algname-IUCB}\xspace}
\newcommand{\roiTS}{\textsc{\algname-RTS}\xspace}
\newcommand{\roiCI}{\textsc{\algname-RCI}\xspace}

\newcommand{\DKBO}{\textsc{DKBO-AE}\xspace}
\newcommand{\LAMCTS}{\textsc{LA-MCTS}\xspace}
\newcommand{\turbo}{\textsc{TuRBO}\xspace}

\newcommand{\funcDot}[0]{\ensuremath{\mathord{\cdotp}}}

\title{Learning Regions of Interest for Bayesian Optimization with Adaptive Level-Set Estimation}

\author{\name Fengxue Zhang \email zhangfx@uchicago.edu \\
       \addr University of Chicago\\
       \AND
       \name Jialin Song \email jialin.utpt@gmail.com \\
       \addr Nvidia\\
       \AND
       \name James Bowden \email jbowden@caltech.edu \\
       \addr California Institute of Technology\\
       \AND
       \name Alexander Ladd \email chenyuxin@uchicago.edu \\
       \addr Lawrence Livermore National Laboratory\\
       \AND
       \name Yisong Yue \email yyue@caltech.edu \\
       \addr California Institute of Technology\\
       \AND
       \name Thomas A. Desautels \email desautels2@llnl.gov \\
       \addr Lawrence Livermore National Laboratory\\
       \AND
       \name Yuxin Chen \email chenyuxin@uchicago.edu \\
       \addr University of Chicago\\
       }
\editor{x}

\begin{document}

\maketitle

\begin{abstract}

%Bayesian optimization (BO) has been applied in black-box optimization tasks which are common in domains including optimal experimental design, self-tuning systems, and hyperparameter optimization. Despite the strong theoretical guarantees for BO in canonical settings, it remains a significant challenge for BO to scale to high-dimensional and non-stationary scenarios. Recent works attempt to exploit the locality of the black-box functions in BO --- by partitioning the search space, these algorithms learn a Gaussian process model over regions of interest that better captures the locality of the black-box function. Various heuristics have been proposed along this direction, most relying on additional hyperparameters to be fine-tuned for specific tasks. (e.g., number of local regions/models to be considered, number of examples in each partition, etc.). 
We study Bayesian optimization (BO) in high-dimensional and non-stationary scenarios.  Existing algorithms for such scenarios typically require extensive hyperparameter tuning, which limits their practical effectiveness.
%Despite the promising performance of Bayesian optimization (BO) in canonical black-box optimization tasks, it remains a significant challenge for BO to scale to high-dimensional and non-stationary scenarios. 
%Recent works attempt to exploit the locality of the black-box functions in BO by partitioning the search space and learning a Gaussian process model over regions of interest. However, most existing heuristics are sensitive to many hyperparameters (e.g., the number of partitions, the number of examples in each partition etc.), which require significant effort to be fine-tuned.
%that better captures the locality of the black-box function. Various heuristics have been proposed along this direction, most relying on additional hyperparameters to be fine-tuned for specific tasks. (e.g., number of local regions/models to be considered, number of examples in each partition, etc.). 
%In this paper, 
We propose a framework, called \algname, which adaptively filters for a high-confidence region of interest (ROI) as a superlevel-set of a nonparametric probabilistic model such as a Gaussian process (GP).  Our approach is easy to tune, and is able to focus on local region of the optimization space that can be tackled by existing BO methods.  The key idea is to use two probabilistic models: a coarse GP to identify the ROI, and a localized GP for optimization within the ROI. We show theoretically that \algname can efficiently shrink the search space, and can exhibit a tighter regret bound than standard BO without ROI filtering.  We demonstrate empirically the effectiveness of \algname on both synthetic and real-world optimization tasks.

\end{abstract}
\section{Introduction}
Bayesian optimization (BO) is a popular statistic-model-based sequential optimization method in various fields of science and engineering, including scientific experimental design \citep{yang2019machine}, robotics planning \citep{Berkenkamp2016SafeOpt,sui2018stagewise}, self-tuning systems \citep{zhang2022learning} and hyperparameter optimization \citep{snoek2012practical}. These applications often involve optimizing a black-box function that is expensive to evaluate, where the statistics-guided efficient optimization algorithm is desired. The common practice in BO is to employ Gaussian processes (GPs) \citep{rasmussen:williams:2006} as a statistic surrogate model for the unknown objective function due to its  mathematical simplicity %%in Bayesian inference 
as well as the promising capability in terms of learning and inference, which allows for defining effective acquisition functions.

Despite strong empirical and theoretical results under certain assumptions (e.g., smoothness) \citep{srinivas2009gaussian,wang2017max,wang2016optimization}, BO has struggled in many real-world settings due to the \emph{high-dimensional}, \emph{large-scale}, and \emph{heterogeneous} nature of optimization tasks. Besides the well-known curse of dimensionality \citep{bengio2005curse}, the heterogeneity and scarcity of training data in real-world tasks make it challenging to fit a single (global) GP for data acquisition \citep{eriksson2019scalable}. Meanwhile, purely relying on local characteristics has been proven to be ineffective for global optimization, due to the ignorance of the correlations on observations that are normally captured by the GP.
The trade-off between exploiting data locality and exploring uncertainty at a global scale emerges as a critical problem in real-world BO settings, especially when the global smoothness assumption no longer holds.

\begin{figure*}[t]
  \centering
  \begin{subfigure}[b]{0.25\linewidth}
    \includegraphics[trim=.5cm -.6cm 1.5cm 0cm, width=\textwidth]{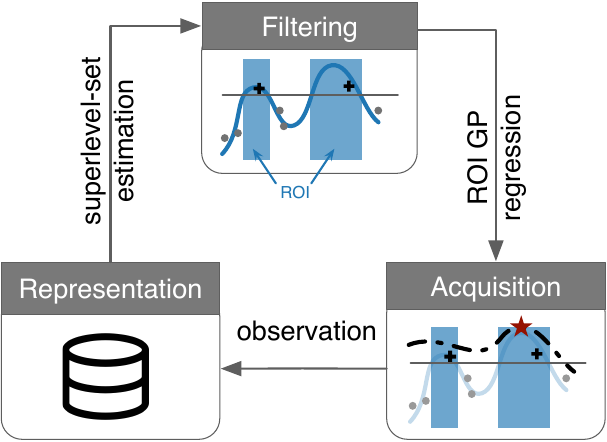}
    \caption{The \algname framework \label{fig:illustration}}
  \end{subfigure}%
  \qquad
  \begin{subfigure}[b]{0.66\linewidth}
    \centering
    \includegraphics[trim=0cm 2.5cm 3cm 0cm, width=\textwidth]{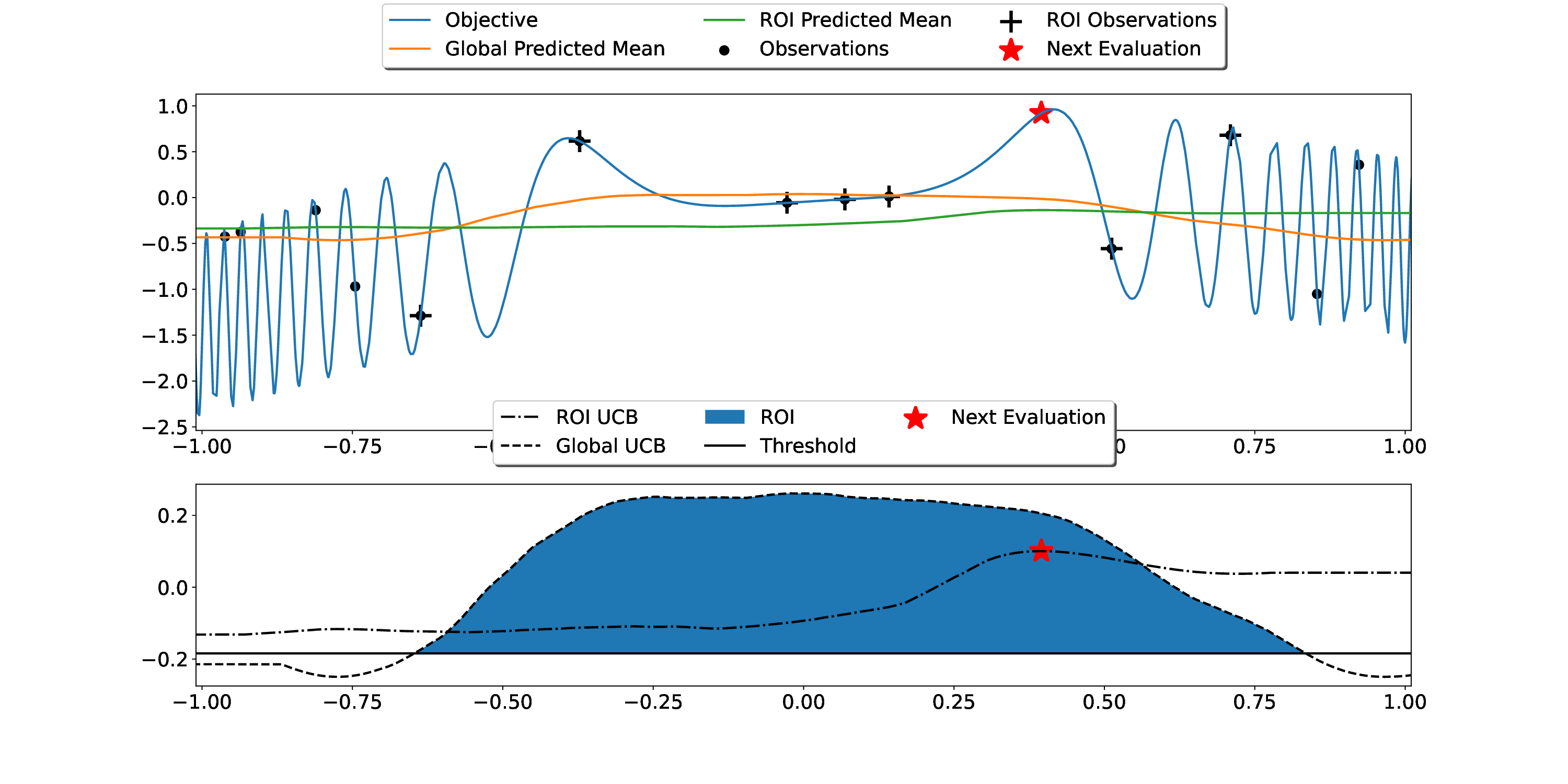}
    \caption{Illustration on a synthetic 1D function. \label{fig:1d_illustration}}
  \end{subfigure}
  \caption{(a) Schematic of the algorithmic framework of \algname. It first identifies the regions of interest by estimating the superlevel-set via a global GP. Then it trains a second GP on the superlevel-set ROIs and uses this GP to acquire the next data point to evaluate (marked by ``\textcolor{Maroon}{$\bigstar$}''). 
    (b) illustrates a single iteration of \algname.
    The upper figure shows the underlying \textcolor{NavyBlue}{objective function}, together with the posterior mean for both the \textcolor{YellowOrange}{global GP} and \textcolor{ForestGreen}{ROI GP}. Training examples (i.e. observations) for the \textcolor{YellowOrange}{global GP} and the \textcolor{ForestGreen}{ROI GP} are marked by ``$\bullet$'' and ``$\RPlus$'' accordingly. The lower figure shows the filtering mechanism using the UCB (\sampleline{dashed}) of the \textcolor{YellowOrange}{global GP} and the threshold determined by the maximal LCB (\sampleline{}) of the \textcolor{ForestGreen}{ROI GP} (\secref{sec:ballet}, \lineref{alg:ln:filtering} of \algref{alg:main}). By learning the \textcolor{ForestGreen}{ROI GP} on the filtered area, \algname is guided by its acquisition function (\sampleline{dash pattern=on .7em off .2em on .05em off .2em}) to the \colorbox{NavyBlue!20}{ROIs}  rather than the ``bad'' regions which, with high confidence, is sub-optimal. 
    The next data point is then chosen from the ROIs (\textcolor{Maroon}{$\bigstar$}).
    Details for this 1D synthetic experiment are provided in \secref{sec:exp}.} \label{fig:overview}
\end{figure*}

Historically, various partitioning-based BO methods have been proposed to tackle this challenge. These methods, often based on certain clustering heuristics, learn the \emph{regions of interest} (ROI) to better reflect the data locality. A common issue for existing heuristics is the added layer of complexity for model fine-tuning, which involves optimizing extra hyperparameters such as the number of ROIs \citep{eriksson2019scalable}, maximum leaf size in the tree-structured partitioning methods, and methods to generalize the partition learned on the accumulated observations to the whole search space \citep{8187198}.

We propose a novel nonparametric approach for partitioning-based BO that demonstrates strong empirical performance in real-world tasks, while having few hyperparameters to maintain. The proposed algorithm is inspired by the \emph{level-set estimation} (LSE) problem, where a level-set corresponds to a set of points for which the black-box objective function takes value above (or below) some \emph{given} threshold. Given a threshold, \citet{10.5555/2540128.2540322} show that one can leverage the point-wise confidence interval to actively identify the level-set with a theoretical guarantee. In the context of Bayesian optimization, the threshold could reduce to the lower confidence bound of the global optima.

\paragraph{Our contribution} Following the above insight, we propose the novel Bayesian optimization framework with adaptive estimation of regions of interest. As illustrated in \Figref{fig:overview}, The algorithm partitions the search space based on confidence intervals and identifies the superlevel-set as the ROIs of high confidence contain the global optimum. We propose a novel acquisition function that relies on both the global model and the ROI model to capture the locality while not sacrificing global knowledge through optimization.
We further provide rigorous theoretical analyses showing that the proposed acquisition function can, in a principled way, exhibit an improved regret bound compared to its canonical BO counterpart without the filtering component. We demonstrate the effectiveness of the proposed framework with an empirical study on several synthetic and real-world optimization tasks.

\section{Related Work}
\paragraph{High-dimensional Bayesian optimization}

BO often uses Gaussian processes as a (mathematically) simple yet powerful tool to parametrize the black-box function. 
However, GPs are difficult to fit and \rebuttal{be applied in optimization} in the high-dimensional setting due to the curse of dimensionality; thus classical BO algorithms need to be modified for high-dimensional function classes \citep{djolonga2013high}. 

A class of methods leverages additional structures, such as additive GPs, to mitigate the challenge of training a single global GP. For instance, LineBO restricts its search space to a one-dimensional subspace with reduced sample complexity at each step \citep{kirschner2019adaptive}. GP-ThreDS relies on Hölder condition on the unknown objective to prune the search space, avoids the discretization cost exponential to the dimensionality of the search space, and speeds up the optimization \citep{salgia2021domain}. In contrast, we aim at the applications where no Hölder smoothness is guaranteed, while valid discrete candidates in the search space are given.

\rebuttal{Another line of work assumes that despite the high dimensionality of the original input, the effective dimensionality is smaller. Therefore, it is feasible to embed the high-dimensional input space into a lower-dimensional subspace using techniques such as random projection and variable selection, while preserving the desired properties for optimization \citep{song2022monte, wang2016bayesian, Letham2020Re, HeSBO19, papenmeier2022increasing}. Additionally, \citet{mcintire2016sparse, moss2023inducing} propose that a reduced set of points can effectively represent the original high-dimensional space without significantly sacrificing uncertainty quantification. They introduce sparse GP as an efficient surrogate for high-throughput Bayesian optimization. Although these works share the spirit of conducting Bayesian optimization on a reduced complexity set compared to the original high-dimensional input space, they can be integrated into any GP-based optimization framework. Hence we do not make comparison with them in experiments.}

\fengxue{Actually the original content could be well summarized by HDBO as a paragraph title. Hence, I keep it while differentiating the works most related to ours from other HDBO methods. and refer to the survey offered by the reviewer for a more comprehensive discussion over HDBO.}
Our proposed method relates most closely to methods with input space partitions \citep{wabersich2016advancing, eriksson2019scalable, wang2020learning, pmlr-v133-sazanovich21a}. Notably, TurBO \citep{eriksson2019scalable} maintains a collection of local GPs and allocates queries with a multi-armed bandit procedure. LA-MCTS \citep{wang2020learning} learns a partition of the input space and uses Monte Carlo tree search (MCTS) to decide a subspace to apply BO. Compared with the proposed \algname, these partitioning methods rely on heuristics and add extra complexity to the optimization task with hyperparameters of these heuristics, e.g., TuRBO relies on the number of trust regions and LA-MCTS relies on leaf size, a hyperparameter in UCB for the subspace selection and one for the partitioning algorithm. \footnote{\rebuttal{Due to the diverse range of interests and fields of work in HDBO, we recommend referring to the recent survey conducted by \cite{10.1145/3545611} for a more extensive discussion on High-dimensional Bayesian Optimization with Gaussian Processes.}}

\paragraph{Partition-based Bayesian active learning and optimization} 
Partition-based methods are common in BO with safety constraints \citep{pmlr-v133-sazanovich21a, sui2018stagewise, makarova2021risk}. These methods %make use of the 
use LCB from GPs to partition the input space into safe and unsafe subspaces. Subsequent optimizaton queries are restricted to the safe subspaces only. Another related work is the level set estimation (LSE) method by \citet{10.5555/2540128.2540322}, where the authors use both UCB and LCB to narrow down regions where a particular function value is likely to exist. A unified framework \rebuttal{TRUVAR} for BO and LSE task \citep{bogunovic2016truncated} proposes a similar filtering method but does not learn a local surrogate model. Instead, the filtering is used to constrain its acquisition function.  Our method {inherits} the spirit of LSE to leverage the confidence interval to adaptively partition the search space. % in a simple yet principled way. 

\paragraph{Partition-based optimization methods}
More broadly speaking, partitioning the input space is a general strategy employed by several optimization methods \citep{munos2011optimistic, 8187198, 7352306, JMLR:v22:18-220, kawaguchi2016global}. Simultaneous optimistic optimization (SOO) algorithm \citep{munos2011optimistic, 8187198}, which is a non-Bayesian approach that intelligently partitions the space based on observed experiments to effectively balance exploration and exploitation of the objective. A modification of SOO, named Locally Oriented Global Optimization (LOGO) \citep{kawaguchi2016global}, achieves both fast convergence in practice and a finite-time error bound in theory. However, 
these non-Bayesian approaches have seen more degraded empirical performance on high-dimensional functions than their Bayesian counterparts \citep{JMLR:v22:18-220}. % likely due to the absence of the potential benefits of posterior inference.

\section{Bayesian Optimization with Adaptive Level-Set Estimation}\label{sec:method}
We consider the standard BO setting for sequentially optimizing a function $f: \searchSpace \rightarrow \reals$, where $\searchSpace \subseteq \reals^d$ is the search space. 
At iteration $t$, we maintain a Gaussian process %$\mathcal{GP}$ 
as the surrogate model,  picks a point $\instance_t\in\searchSpace$ by maximizing the acquisition function $\alpha: \searchSpace \rightarrow \reals$, 
and observe the function value perturbed by additive noise: 
$y_t = f(\instance_t) + \epsilon_t$ with $\epsilon_t \sim \mathcal{N}(0, \sigma^2)$ 
being i.i.d. Gaussian noise. The goal is to maximize the sum of rewards 
$\sum^T_{t=1}f(\instance_t)$ over $T$ iterations, or equivalently, 
to minimize the \emph{cumulative regret} $R_T \triangleq \sum_{t=1}^T r_t$, 
where $r_t \triangleq \max\limits_{\instance\in \searchSpace}f(\instance) - f(\instance_t)$ 
denotes the \textit{instantaneous regret}. Another common performance metric in BO is the \textit{simple regret} $r^*_T \triangleq \max\limits_{\instance\in \searchSpace}f(\instance) - \max\limits_{t\leq T}f(\instance_t)$. 

\subsection{The \algname framework}\label{sec:ballet}
\paragraph{Global modeling and representation} Existing works use heuristics to partition the historical observations $\Selected_t = \{X_t, Y_t\}$ first and then generalize it to the whole search space $\searchSpace$ \citep{wang2020learning, eriksson2019scalable}. Here $Y_t =\{y_1,...,y_t\}$ and $X_t =\{\instance_1,...,\instance_t\}$. %The heuristics could introduce unnecessary complexity and potentially incur the loss on the accuracy of the partitioning. 
The heuristics could be sensitive to additional hyperparameters of the partitioning model (e.g., number of partitions, etc), which in turn affect the optimization performance.

Instead, we propose to learn a partitioning on $\searchSpace$ with a global estimation of the underlying blackbox function $\globalf \defeq f$, which is modeled by a Gaussian process $\mathcal{GP}_{\globalf}(m_{\globalf}(\instance), k_{\globalf}(\instance, \instance'))$ trained on the historical observations. $\mathcal{GP}_{\globalf}$ is parameterized by $\algParam_{\globalf}$, where $m_{\globalf}(\instance)$ is the mean function and $k_{\globalf}(\instance, \instance')$ is the covariance function.

In this work, we resort to \emph{Deep Kernel Learning} (DKL) \citep{pmlr-v51-wilson16} as a scalable tool to train the GPs\footnote{We propose a kernel-agnostic framework and in implementation, we apply the efficient deep kernel for large-scale optimization. In deep kernel learning, which is shown to bear strong empirical performance in regression and optimization (e.g.\citep{pmlr-v51-wilson16, wistuba2021few}), 
the learning cost is $\mathcal{O}(n)$ for $n$ training points, and the prediction cost is
$\mathcal{O}(1)$ per test point and is more efficient than exact GP in terms of computational cost.}.

The algorithm learns a latent space mapping $\featureExtractor: \searchSpace\rightarrow \LatentRepSet$ on a neural network to convert the input space $\searchSpace$ to the latent space $\LatentRepSet$, and constructs an objective mapping $h: \LatentRepSet \rightarrow \reals$ such that $f(\instance)\approx h(\featureExtractor(\instance)),\ {\forall \instance\in \searchSpace}$. 
The neural network $\featureExtractor$ and the base kernel $k$ together are regarded as a \emph{deep kernel}, 
denoted by $k_{\globalf}(\instance, \instance') = k\paren{\featureExtractor(\instance), \featureExtractor(\instance')}$ \citep{pmlr-v51-wilson16}. The deep kernel is trained by maximizing the negative log-likelihood (NLL) $-\log(\Pr{\mathbf{y}_t \mid X_t,\algParam_{\globalf,t}}) = -\frac{1}{2}\mathbf{y}_t^\top(\GramMat_{\globalf,t}+\sigma^2I)^{-1}\by_t -\frac{1}{2}\log|(\GramMat_{\globalf,t}+\sigma^2I)| -\frac{t}{2}\log(t)$ which is the learning objective for the kernel \citep{rasmussen:williams:2006}. \rebuttal{Note that DKL relies on KISS-GP \cite{wilson2015kernel} which generalize inducing point methods with kernel interpolation for efficient inference and is related to the sparse GP methods \cite{mcintire2016sparse, moss2023inducing}}. In addition to DKL, we use the unlabeled dataset sampled from $\searchSpace$ to pre-train an Auto-Encoder and use the parameters of its encoder to initialize the neural network $\featureExtractor$ following the protocol described by \citet{ferreira2020using}.

At iteration $t$, given the selected points $\Selected_t$, the posterior over $\globalf$ also takes the form of a GP, with mean $\mu_{\globalf,t}(\instance) = k_{\globalf,t}(\instance)^\top(\GramMat_{\globalf,t}+\sigma^2I)^{-1}\by_t$ and covariance $k_{\globalf,t}(\instance, \instance') = k_{\globalf}(\instance,\instance')-k_{\globalf,t}(\instance)^\top(\GramMat_{\globalf,t}+\sigma^2I)^{-1}k_{\globalf,t}(\instance')$, where $k_{\globalf,t}(\instance) \triangleq \bracket{k_{\globalf}(\instance_1, \instance),\dots, k_{\globalf}(\instance_t, \instance)}^\top$ and $\GramMat_{\globalf,t} \triangleq \bracket{k_{\globalf}(\instance, \instance')}_{\instance,\instance' \in \Selected_t}$ is the positive definite kernel matrix \citep{rasmussen:williams:2006}.

\begin{algorithm*}%[H]
\caption{\textbf{\underline{B}}ayesian Optimization with \textbf{\underline{A}}daptive \textbf{\underline{L}}eve\textbf{\underline{l}}-Set \textbf{\underline{E}}s\textbf{\underline{t}}imation (\algname)}
\label{alg:main}

    \begin{algorithmic}[1]
        \STATE {\bf Input}:Search space $\searchSpace$, initial observation $\Selected_0$, horizon $T$;
        \FOR{$t = 1\ to\ T$}
            \STATE Fit the global Gaussian process $\mathcal{GP}_{\globalf,t}$: $\algParam_{\globalf,t} \leftarrow \argmax_{\algParam_{\globalf}}-\log \Pr{Y_t\mid X_{t-1},\algParam_{\globalf}}$ 
             
            \STATE %Partition by region of interest filtering: 
            Identify ROIs via superlevel-set estimation
            $\roi_t \leftarrow \{\instance\in \searchSpace \mid    \UCBit_{\globalf, t}(\instance) \geq  \LCB_{\globalf, t, \max}\}$ \label{alg:ln:filtering}
            
            \STATE Partition the historical observation: $\hat{\Selected}_t \leftarrow \{(\instance,y)\in \Selected_t \mid \instance \in {\roi_t}\}$. 
             
            \STATE  Fit the ROI Gaussian process  $\mathcal{GP}_{\roif,t}$: $\algParam_{\roif,t} \leftarrow \argmax_{\algParam_{\roif}}-\log \Pr{Y_t \cap \hat{\Selected_t}|{X_{t} \cap \hat{\Selected}_t},\algParam_{\roif}}$ 
          
            \STATE Optimize the {superlevel-set} acquisition function: $\instance_{t+1} \leftarrow \argmax\limits_{\instance\in \roi}{\acq(\instance)}$ (e.g., as defined in \eqref{eq:acqCI},~\ref{eq:acqROI-ts} or \ref{eq:acqROI-ci})
            
            \STATE $\Selected_{t+1} \leftarrow \Selected_{t} \cup \{(\instance_{t+1}, y_{t+1})\}$
        \ENDFOR
        \STATE {\bf Output}: $\max\limits_{t}{y_t}$
    \end{algorithmic}
\end{algorithm*}

\paragraph{Superlevel-set estimation and filtering} The global $\mathcal{GP}_{\globalf}$ induces a filter on $\searchSpace$ to locate the region of interest $\roi$. It is desired for $\roi$ that with high probability, the optimum $\instance^* \in \argmax_{\instance\in\searchSpace} f(\instance)$ is contained in $\roi$, while $\Vert\roi\Vert \ll \Vert\searchSpace\Vert$. 

Specifically, we leverage the confidence interval of the global Gaussian process $\mathcal{GP}_{\globalf}$ to define the upper confidence bound 
$  \UCBit_{\globalf, t}(\instance) \triangleq \mu_{\globalf,t-1}(\instance) + \beta^{1/2}_{t}\sigma_{\globalf,t-1}(\instance)$ 
and lower confidence bound $ \LCB_{\globalf, t}(\instance) \triangleq\mu_{\globalf,t-1}(\instance) - \beta^{1/2}_{t}\sigma_{\globalf,t-1}(\instance)$, where $\sigma_{\globalf,t-1}(\instance) = k_{\globalf,t-1}(\instance,\instance)^{1/2}$ and $\beta$ acts as an scaling factor. Then the maximum of the global lower confidence bound $ \LCB_{\globalf, t, \max} \triangleq \max_{\instance\in\searchSpace}  \LCB_{\globalf, t}(\instance)$ can be used as the threshold, and we attain the superlevel-set
\begin{align}
    \roi_t \triangleq \curlybracket{ \instance \in \searchSpace \mid \UCBit_{\globalf, t}(\instance) \geq  \LCB_{\globalf, t, \max}} \label{eq:roi}
\end{align} 
as the region(s) of interest. The historical observation on this subset is denoted as
\begin{align}
    \hat{\Selected}_t \triangleq \curlybracket{(\instance,y)\in \Selected \mid \instance \in \roi_t}. \label{eq:selected}
\end{align}

\begin{rem}
\algname is not assuming that the resulting $\roi$ is composed of one single cluster. 
\algname learns a single GP over $\roi_t$ and optimizes on all these localities at the same time. Intuitively it aims at conducting (local) BO on the top tier (which could consist of multiple regions) of the unknown function. This mechanism avoids being overconfident to identify only one region of interest or the need to manually specify the number of clusters beforehand. Here we use the term ``superlevel-set'' 
to differentiate from the methods conducting local BO.
\end{rem}

\subsection{\interCI}\label{sec:interci}

The goal of the filtering step in \algname is to shrink the search space $\roi_t$ while ensuring that the optimum is contained in the ROIs with high probability.
By definition of $\roi_{t}$ (\eqref{eq:roi}), we note that at iteration $t$ the size of the search space $\Vert\roi_{t}\Vert$ is directly affected by $\UCBit_{\globalf, t}(\instance)$. \footnote{\rebuttal{Given any finite discretization $\discreteSet$ of $\searchSpace$, we refer the size of the search space as the cardinality  $\vert \discreteSet \vert$.}} We thus consider the width of the range of $\UCBit_{\globalf, t}$ over $\instance \in \roi_{t}$, formally defined as

\rebuttal{
\begin{align}
    \roiRangeAt{t}(\instance_t) \triangleq \max_{\instance \in \roi} \UCBit_{\globalf, t}(\instance) - \LCB_{\globalf, t, \max} \label{eq:roirange}
\end{align}
}
as a surrogate objective to minimize. 

\subsubsection{%Confidence-interval-based 
Acquisition function} 
Evaluating \eqref{eq:roirange} for a new data point $\instance$ requires 1-step look-ahead (i.e., computing the expected $\roiRangeAt{t+1}$ should $\instance$ be acquired at $t$), which could be expensive. Instead, we consider the \emph{point-wise confidence interval} of the ROI Gaussian process $\mathcal{GP}_{\roif}$ trained on $\hat{\Selected}$, denoted by 
\begin{align}
    \CI_t(\instance) \triangleq \bracket{\LCB_{\roif,t}(\instance),\UCBit_{\roif, t}(\instance)}, \label{eq:roiciatt}
\end{align}
and simply use the width of $|\CI_t(\instance)|$ as an efficiently proxy for evaluating $\instance$. 

\paragraph{Mitigating the loss of information of $\mathcal{GP}_{\roif}$} At each iteration, 
\algname conducts superlevel-set estimation %region of interest filtering 
and then runs BO on $\roi$ using the ROI Gaussian process $\mathcal{GP}_{\roif}$ trained on $\hat{\Selected}$. 
Note that $\mathcal{GP}_{\roif}$ could better capture the locality at the cost of \emph{losing partial historical observations} due to the filtering as $\hat{\Selected}\subseteq \Selected$. The missing historical observations could result in additional undesired uncertainty in $\mathcal{GP}_{\roif}$ compared to the global GP $\mathcal{GP}_{\globalf}$.
To avoid such information loss while taking the advantage of the identified ROIs,

we propose to tighten the confidence interval (\eqref{eq:roiciatt}) 
by taking the intersection of the confidence intervals 
from all ROI GPs trained from each of the previous iterations $\mathcal{GP}_{\roif, i\leq t}$ and the corresponding global GPs, $\mathcal{GP}_{\globalf, i\leq t}$. In this way, the acting superlevel-set confidence interval would be $\hatCI_{t}(\instance) \triangleq \bracket{\acqLCB_{t}(\instance), \widehat{  \UCBit}_{t}(\instance)}$ , where 

\begin{align}
\begin{cases}
    \acqLCB_{t}(\instance) \triangleq \max_{i\leq t, f\in\{\roif, \globalf\}}{ \LCB_{f, i}(\instance)} %\label{eq:acqLCB} 
    \\
    \acqUCB_{t}(\instance) \triangleq \min_{i\leq t, f\in\{\roif, \globalf\}}{  \UCBit_{f, i}(\instance)} %\label{eq:acqUCB}
\end{cases} \label{eq:acqLCB+UCB}
\end{align}
It is possible that the intersection in \eqref{eq:acqLCB+UCB} results in an empty CI due to the dynamics brought by the learned kernels. 
In practice, instead of taking the intersections of all the historical GPs, we could mitigate the problem by only taking the intersection of the CIs at step $t$ to get $\tildeCI_{t}(\instance) = [\acqLCBt_{t}(\instance), \acqUCBt_{t}(\instance)]$. Here 

\begin{align}
\begin{cases}
    \acqLCBt_{t}(\instance) \triangleq \max_{f\in\{\roif, \globalf\}}{ \LCB_{f, t}(\instance)} %\label{eq:acqLCBt} 
    \\
    \acqUCBt_{t}(\instance) \triangleq \min_{f\in\{\roif, \globalf\}}{  \UCBit_{f, t}(\instance)} %\label{eq:acqUCBt}
\end{cases} \label{eq:acqLCBt+UCBt}
\end{align}

Note when $\LCB_{\roif,t} \leq \LCB_{\globalf,t}$ and $\LCB$ is monotonically increasing wrt $t$
, it holds that $\LCB_{\globalf, t, \max}=\max_{\instance \in \hat{\searchSpace}}\acqLCB_{ t}(\instance)=\max_{\instance \in \hat{\searchSpace}}\acqLCBt_{ t}(\instance)$.
 
\paragraph{The \interCI acquisition function} We propose to apply the \underline{i}ntersection of the \underline{c}onfidence \underline{i}ntervals as an acquisition function for \algname (\interCI), namely 
\begin{equation}\label{eq:acqCI}
    \acq(\instance) \triangleq \acqUCBt_{t}(\instance) - \acqLCBt_{t}(\instance)  
\end{equation}

Our algorithm is presented in \algoref{alg:main}. In the following subsection, we rigorously justify the use of \eqref{eq:acqCI} as our acquisition function, and prove that the cost on the optimization performance using the relaxation from \eqref{eq:acqLCB+UCB} to \eqref{eq:acqLCBt+UCBt} could be bounded under certain conditions.

\subsubsection{Theoretical analysis}
By abuse of notation, we let the maximum confidence interval on a certain set denoted by 

$$\hatCI_{t,\max}(\funcDot) =  \bracket{\max_{\instance \in \funcDot}\acqLCB_{ t}(\instance), \max_{\instance \in \funcDot}\acqUCB_{ t}(\instance)}$$ 

$$\tildeCI_{t,\max}(\funcDot) = \bracket{\max_{\instance \in \funcDot}\acqLCBt_{t}(\instance),
\max_{\instance \in \funcDot}\acqUCBt_{t}(\instance)}$$

The following lemma shows that the interval
$\hatCI_{t,\max}(\hat{\searchSpace})$ is a high confidence interval for $f^* = \max\limits_{\instance\in \searchSpace}f(\instance)$ given a good discretization of the search space.

\begin{lem}\label{lem:CI}
\todo{avoid validity concern over the assumptions}
Assume $\forall t < T, \instance\in\searchSpace$, $f(\instance)$ is a sample from global $\mathcal{GP}_{\globalf, t}$. \footnote{\rebuttal{Here rigorously $\forall t < T$, $\mathcal{GP}_{\globalf, t}$ should share the same prior with each other, while we periodically retrain the model similar to the practice of \cite{tripp2020sample} and we reflect it with the subscript $t$. }}
For any $\delta \in (0,1)$ and any finite discretization $\discreteSet$ of $\searchSpace$ 
containing the optimum $\instance^* = \argmax_{\instance\in \searchSpace}f(\instance)$
, with $\beta_t=2\log(2\vert \discreteSet \vert \pi_t/ \delta)$ where $\sum_{t\geq 1}\pi_t^{-1} = 1$, 
$\Pr{f^* \in \hatCI_{t,\max}(\discreteSet)} \geq 1-\delta$.

\end{lem}

A proper choice of $\pi_t$ satisfying Lemma 1 is $\pi_t = \frac{\pi^2t^2}{6}$. The following corollary shows that with high probability, the global optimum is contained in the interval.

\begin{rem}
Rigorously, $\forall t < T, \instance\in\hat{\searchSpace}$, the marginalized $\mathcal{GP}_{\roif, t}$ and  $\mathcal{GP}_{\globalf, t}$ shall be the same. Therefore, $\forall t < T, \instance\in\hat{\searchSpace}$, $f(\instance)$ is a sample from $\mathcal{GP}_{\roif, t}$ as well. However, in practice, it is challenging to specify the ideal prior. We introduce $\mathcal{GP}_{\roif, t}$ into the analysis to reflect the benefits of learning the hyperparameters for each GP separately in real-world scenarios. 
\end{rem}

\begin{cor}\label{cor:CI}
With the same conditions as in \lemref{lem:CI}, $\Pr{\instance^* \in \roi_t} \geq 1-\delta, \forall t\geq 1$.
\end{cor}
For simplification, we use the notation $\discreteROI = \discreteSet \cap \roi$. Taking the union bound over \lemref{lem:CI} and \corref{cor:CI}, we obtain the following result:
\begin{cor} \label{cor:CI2} 
With probability at least $1-2\delta$, the global optimum lies in the following interval $\hatCI_{t,\max}(\discreteROI) \subseteq \tildeCI_{t,\max}(\discreteROI)$.

\end{cor}
\corref{cor:CI2} indicates that by narrowing the interval, we could achieve efficient filtering in BALLET and identify the near-optimal areas. Define the maximum information gain about unknown function $f$ after $T$ rounds  as $\maxInfo_{f, T} = \max_{\actionSet\subset \discreteSet: \vert \actionSet \vert=T}{\mutualinfo{y_\actionSet; f_\actionSet}}$. Also, define 
$$\widehat{\maxInfo_T} = \min_{f \in \{\globalf, \roif\}}{\maxInfo_{f, T}}.$$
The following results shows that $%\acq(\instance) 
|\hatCI_{t}(\instance)| = \acqUCB_{t}(\instance) - \acqLCB_{t}(\instance)$ serves the purpose of efficiently narrowing the interval and the resulting range of it is bounded by $\maxInfoBALLET_{T}$.

\begin{prop}\label{prop:regret}
Under the same conditions assumed in \lemref{lem:CI} except for $\beta_t=2\log(2\vert \discreteROI \vert \pi_t/ \delta)$, with acquisition function $%\acq(\instance) = 
|\widehat{ \CI}_{t}(\instance)| = \acqUCB_{t}(\instance) - \acqLCB_{t}(\instance)$, after at most $T \geq \frac{\beta_T \maxInfoBALLET_T C_1}{\epsilon^2}$ iterations, 
$\Pr{|\hatCI_{T,\max}(\discreteROI)| \leq \epsilon} \geq 1 - 2\delta$. Here  $C_1=8/\log(1+\sigma^{-2})$.

\end{prop}

The proposition reveals two potential improvements over the global GP-UCB \citep{srinivas2009gaussian} %\yuxin{over which bound?} 
on the regret bound brought by \algname. First, $\beta_T$ takes smaller value due to the filtering compared to $\lemref{lem:CI}$ which is also the term in the regret bounds of \cite{srinivas2009gaussian}. Second, $\widehat{\maxInfo_T}$ could potentially be smaller than the global $\maxInfo_{\globalf, T}$ with proper kernel learning on ROI. The following corollary shows the cost of using $\acqUCBt_{t}(\instance) - \acqLCBt_{t}(\instance)$ as acquisition function is $C_2^2$ compared to $\acqUCB_{t}(\instance) - \acqLCB_{t}(\instance)$.

\begin{cor} \label{cor:relax_regret} 
Under the same conditions assumed in \lemref{lem:CI} except for $\beta_t=2\log(2\vert \discreteROI \vert \pi_t/ \delta)$, with acquisition function $\acq(\instance) = |\tildeCI_t(x)|= \acqUCBt_{t}(\instance) - \acqLCBt_{t}(\instance)$, after at most $T \geq \frac{ \beta_T \maxInfoBALLET_T C_1 C_2^2}{\epsilon^2}$ iterations, $\Pr{|\tildeCI_{T,\max}(\discreteROI)| \leq \epsilon} \geq 1 - 2\delta$. Here  $C_1=8/\log(1+\sigma^{-2})$, and $C_2 = \frac{\min_{t\leq T}(|\tildeCI_{t,\max}(\discreteROI)|) }{|\tildeCI_{T,\max}(\discreteROI)|}$.

\end{cor}

The proof of \corref{cor:relax_regret} follows the proof of \propref{prop:regret} except leveraging the fact $$\min_{t\leq T}(|\tildeCI_{t,\max}(\discreteROI)|) = C_2|\tildeCI_{T,\max}(\discreteROI)|$$ at its last step.

\subsection{Other \algname variants}
\algname provides a flexible framework for partitioning-based BO. In addition to \interCI, one can run Thompson sampling on ROI as the acquisition function (\roiTS), namely
\begin{equation}\label{eq:acqROI-ts}
    \acqTs(\instance) \triangleq \roif_{t}(\instance) 
\end{equation}
where $\roif_{t} \sim GP_{\roif, t}$.
Another \algname variant is to directly run uncertainty sampling with $\mathcal{GP}_{\roif}$ on $\roi$ %take the width of the confidence interval of on ROI as acquisition function 
(\algname-RCI, where RCI is short for ``ROI-CI''):
\begin{equation}\label{eq:acqROI-ci}
    \acqRoici(\instance) \triangleq |\CI_t(\instance)| = \UCBit_{\roif, t}(\instance) -  \LCB_{\roif, t}(\instance) 
\end{equation}

Compared with \roiTS and \roiCI, the intersection of CIs defined in \eqref{eq:acqCI} in \interCI leverages the posterior information of both $\mathcal{GP}_{\roif}$ and $\mathcal{GP}_{\globalf}$. This allows \interCI to efficiently narrow the confidence interval for $f^*$ by explicitly balancing exploration and exploitation, and achieve a high-probability theoretical guarantee on its optimization performance. We also discuss taking the \UCB of the intersection of CI's as the acquisition function (\interUCB) in the appendix.

\clearpage
\section{Experiment}\label{sec:exp}
\paragraph{Experimental setup}

\begin{figure*}[t]
        \centering
            {
          \includegraphics[trim={0pt 0pt 0pt 0pt}, width=.95\textwidth]{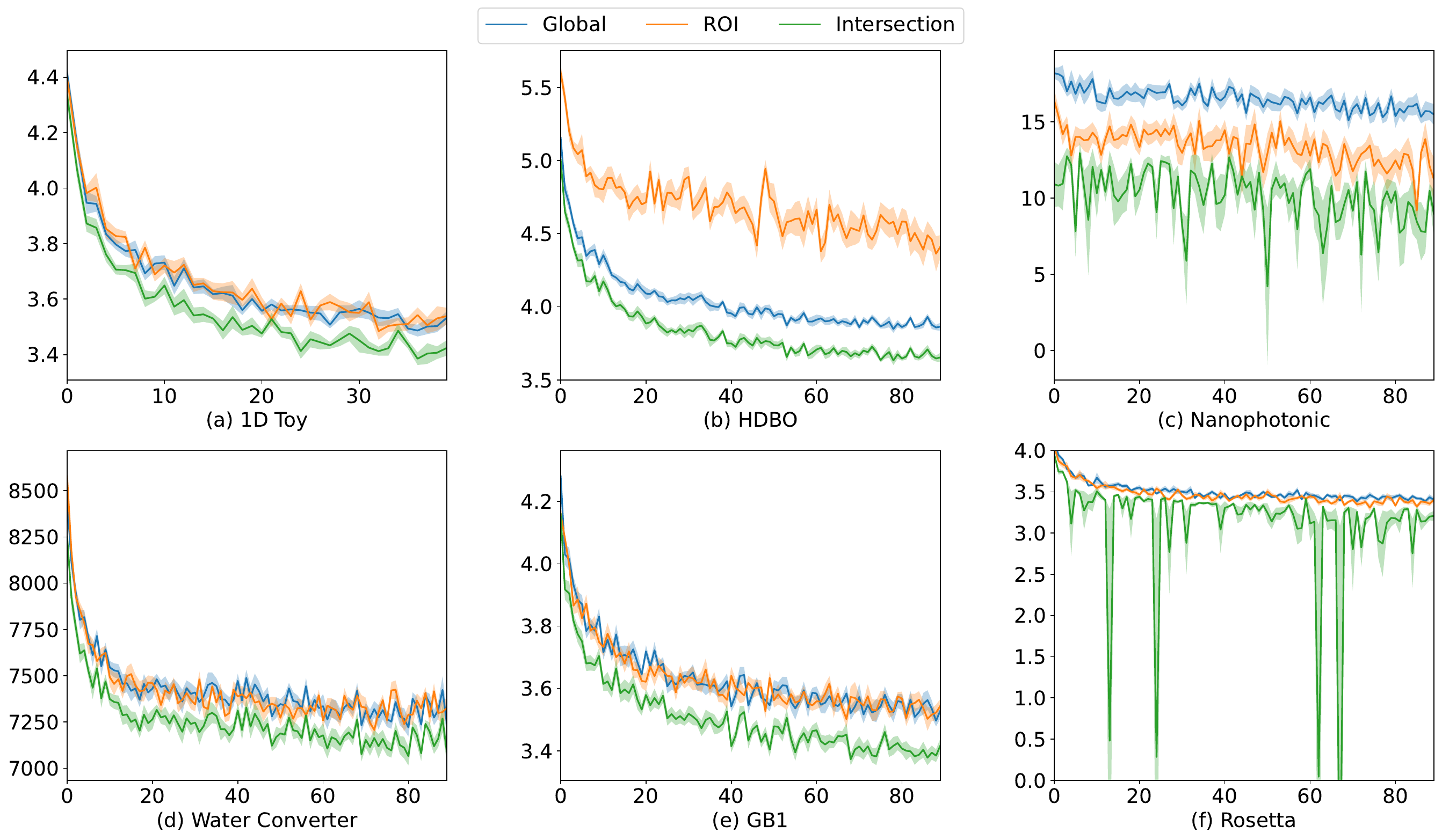}
        }
  \caption{The confidence interval of $f^*$ defined in \corref{cor:CI2}. % on each task is shown here. 
  The results from each task are collected from at least 10 independent trials. The error bar demonstrates the standard error. As $\beta$ varies on different iterations and different search space sizes, we fix $\beta_t=2$ for comparable illustration. The x-axis denotes the number of iterations, and the y-axis denotes the width of the confidence interval.
  }   
  \label{fig:exp:CI}
\end{figure*}

We compare three baseline algorithms in our experiments against \interCI, \roiCI, and \roiTS. The Deep-Kernel-based Bayesian Optimization initialized with a pre-trained AutoEncoder (DKBO-AE) applies the deep kernel where a pre-trained AutoEncoder \footnote{The AutoEncoder is trained with random unlabeled samples.} initializes the neural network $q$ {\citep{zhang2022learning}}. The neural network consists of three hidden layers with 1000, 500, and 50 neurons, and ReLU non-linearity respectively. {The output layer is one-dimensional.} We use squared exponential kernel or linear kernel as the base kernel, i.e. $k_{\SE}(\instance,\instance')=\sigma^2_{\SE}\exp(-\frac{(\instance-\instance')^2}{2l})$ or $k_{\Linear}(\instance,\instance')=\sigma^2_{\Linear}(\instance^T\instance)$, for the deep kernel, and Thompson Sampling \citep{chapelle2011empirical} for the acquisition function $\alpha$. Two other partition-based BO algorithms LA-MCTS \citep{wang2020learning} and TuRBO \citep{eriksson2019scalable} serve as the baselines. Note that DKBO-AE is used as the subroutine for LA-MCTS, {TuRBO-DK}, and \roiTS. The neural network architecture, base kernel and acquisition function are the same. 
\roiCI and \interCI also share the same deep kernel except for applying different acquisition functions. The comparison between \roiTS and DKBO-AE serves as the ablation study of \textbf{the proposed partitioning method}. The comparison between \roiCI and \interCI also serves as the ablation study of \textbf{taking the intersection of CI} as defined in \eqref{eq:acqLCBt+UCBt}. 

One crucial problem in practice is tuning the hyperparameters. For each of the algorithms, the same 10 randomly picked points serve as the warm-up set. We take the default hyperparameters from the open-sourced LA-MCTS \footnote{\scriptsize \url{https://github.com/facebookresearch/LaMCTS}} and TuRBO \footnote{\scriptsize\url{https://botorch.org/tutorials/turbo_1}} implementation. Note that we choose TuRBO-1 implementation for TuRBO where there is one trust region through the optimization, as previous work has shown its robust performance in various tasks \citep{eriksson2019scalable}. For \interCI, we set $\delta$ in \lemref{lem:CI} to be 0.2. In addition, we find that using the $\beta^{1/2}_t$ in \lemref{lem:CI} to identify the ROI could be over-conservative in that it can not filter many areas and let \interCI regress to DKBO-AE with two similar GPs. Through the experiments, we fix $\beta^{1/2}_t=0.2$ only when identifying ROIs as in \lineref{alg:ln:filtering} of \algoref{alg:main}. For all the tested algorithms, the base kernels are squared exponential kernels except for Nanophotonics and Water Converter where we applied linear kernels as the base kernel. We defer the detailed study of parameter choices in \interCI to the appendix.

\begin{figure*}[t]
        \centering
        {
          \includegraphics[trim={0pt 0pt 0pt 0pt}, width=.95\textwidth]{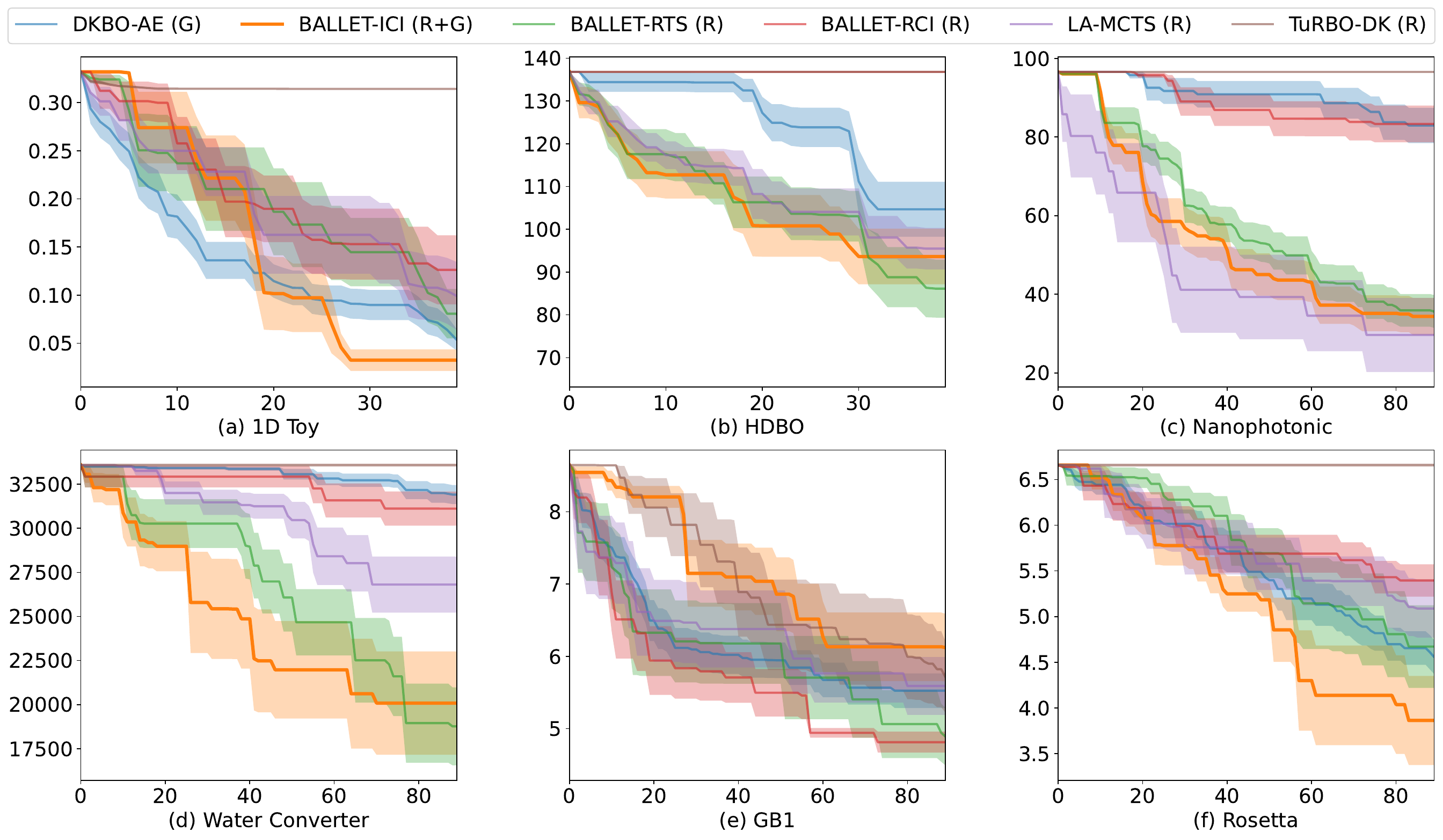}
        } 
  \caption{Simulation results.
  The results from each task are collected from at least 10 independent trials. The error bar demonstrates the standard error. The x-axis denotes the number of iterations, and the y-axis denotes the simple regret. The simple regrets for the 10 initial randomly picked warm-up datasets are clipped. (G), (R), and (R+G) means the global model only, the ROI model only, and the ROI model combined with the global model correspondingly. }   
  \label{fig:exp:simpleRegret}
  \vspace{-6mm}
\end{figure*}

\paragraph{Datasets}
We study the performance of \algname on two synthetic tasks and four pre-collected datasets described below.

\subparagraph{1D-Toy.}
We create a synthetic dataset 1D-Toy of one dimension to illustrate the process of \interCI as is shown in \secref{sec:method}. The function is defined on $\instance\in[-1,1]$ as $f(\instance) = \sin(64\vert{\instance}\vert^ 4) - (\instance-0.2)^2$. This toy function consists of two high-frequency areas on both sides and a low-frequency area in the middle. The neural network is pre-trained on 100 data points.

\subparagraph{HDBO-200D.}
We create a synthetic dataset Sum-200D of 200 dimensions. Each dimension is independently sampled from a standard normal distribution to maximize the uncertainty on that dimension and examine the algorithm's capability to solve the medium-dimensional problem. We want to maximize the label $f(\instance) = \sum^{200}_{i=1}{e^{x_i}}$ which bears an additive structure and of non-linearity. The neural network is pre-trained on 100 data points. 

\subparagraph{Water Converter Configuration-{32D}.} 
This UCI dataset we use consists of positions and absorbed power outputs of wave energy converters (WECs) from the southern coast of Sydney. The applied converter model is a fully submerged three-tether converter called CETO. 16 WECs 2D-coordinates are placed and optimized in a size-constrained environment. Note its values are at the order of $O(10^6)$.

\subparagraph{Nanophotonics Structure Design{-5D}.}
We wish to optimize a weighted figure of merit quantifying the fitness of the transmission spectrum for hyperspectral imaging as assessed by a numerical solver~\citep{song:18.2}. This problem has a 5-dimensional input corresponding to the physical design dimensions of a potential filter. 
Although the input is not high-dimensional, the function represents a discrete solution of Maxwell's equations and has a complex value landscape. 

\subparagraph{GB1-{118D}.}
We use a protein dataset in which the objective is to maximize stability fitness predictions for the Guanine nucleotide-binding protein GB1 given different sequence mutations in a target
region of 4 residues \citep{wu2019machine}. Specifically, we use the ESM embedding generated by a transformer protein language model \citep{rives2021biological}.

\subparagraph{Rosetta Protein Design-{86D}.} 
We use another protein engineering dataset describing a set of antigen/antibody binding calculations. These calculations, executed using supercomputing resources, estimate the change in binding free energy at the interface between each of 71769 modified antibodies and the SARS-CoV-2 spike protein, as compared to the single reference antibody from which they are derived. Estimations of binding free energy ($\Delta\Delta G$) are calculated using protein-structure-based Rosetta Flex simulation software \citep{das2008macromolecular, barlow2018flex}.  These calculations took several CPU hours each and were produced during an antibody design process \citep{desautels2020rapid,desautels2022computationally}.  
Inputs are described with an 80-dimensional feature vector that, relative to the reference sequence, describes changes in the interface between the antibody and the corresponding target region on the SARS-CoV-2 spike. 
This is a particularly relevant problem setting when trying to rapidly choose antibody candidates to respond to a new disease in a timely fashion.

\paragraph{Confidence Intervals}
As shown in \figref{fig:exp:CI}, the CIs of $f^*$ through the optimization of \interCI do not constantly narrow. Instead, on Nanophotonics the width generally remains the same through the 90 iterations, indicating the challenge of fitting these datasets with limited data points and therefore optimizing it with the underfitted GPs. 

The intersection CI is consistently narrower than the other two, where there is no consistent superiority against each other. Though on HDBO the ROI curve is above the global curve, the resulting intersection CI still improves upon global CI, indicating that the maximizer of $\acqUCBt_{t}$ and $\acqLCBt_{t}$ are different for global GP and ROI GP.

The dynamics of kernel learning results in empty intersections on Rosetta, where occasionally the width of CI for $f^*$ turns out to be zero, showing the potential problem in taking the intersection of all historical CIs. Future improvement on \interCI could be better aligning the CI of both ROI and global GPs through different iterations to allow taking the intersection of all historical CIs as in \propref{prop:regret}.

The intersection curve converges faster to non-zero values on both HDBO and GB1 showing the benefits of taking the intersection of global CI and ROI CI as it better captures the localities with ROI GP while not losing information of global GP. However, the width of CI could not directly serve as the indicator for the optimization performance. On GB1, the intersection curve is uniformly better than both the ROI and global CI, while in \figref{fig:exp:simpleRegret}, \interCI does not outperform DKBO-AE, \roiTS, or \algname-ROI-UCB as the CI for $f^*$ is still larger than 3.3.

% TD: Improve the table.
\newcommand\tableW{3.6cm}
\newcommand\tableCW{1.4cm}
\newcolumntype{P}[1]{>{\centering\arraybackslash}p{#1}}
\newcolumntype{M}[1]{>{\centering\arraybackslash}m{#1}}
\begin{table*}[ht]
\renewcommand*{\arraystretch}{2}
\centering
\resizebox{\linewidth}{!}{%
% \begin{tabular}{ |M{1.7cm}||M{2cm}|M{1.6cm}|M{\tableCW}|M{2.3cm}| M{\tableW}|M{3.0cm}| }
\begin{tabular}{ |M{1.7cm}||M{2cm}|M{2cm}|M{2cm}|M{\tableW}| M{\tableW}|M{\tableW}| }
 % \hline
 % \multicolumn{7}{|c|}{Method List} \\
 \hline
    \textbf{Methods} & DKBO-AE & LA-MCTS & TuRBO & \roiTS & \roiCI & \interCI  \\
 \hline\hline
    \textbf{BO Model} & Global & Local & Local & Local & Local & Global + Local \\
 \hline
    \textbf{Acquisition Function} & $f_{g,t} \sim GP_{f_g, t}$   & \multicolumn{3}{c|}{$\roif_{t} \sim GP_{\roif, t}$}  &  $ \UCBit_{\roif, t}(\instance) -  \LCB_{\roif, t}(\instance)$ &
    $\acqUCBt_{t}(\instance) - \acqLCBt_{t}(\instance)$\\
 \hline
    \textbf{Average Ranking} &   3.67  & 3.33   & 5.83 & \textbf{2.00} &  3.50 & \textbf{2.33}\\

 \hline
\end{tabular}}
\small
\caption{\rebuttal{Comparison of different methods tested in the experiments. The BO model row shows the model on which the acquisition function is defined. LA-MCTS relies on global Monte-Carlo tree search, and both TuRBO and \roiTS rely on a global GP to identify the ROI/Trust Region, despite being tagged as `local' for the BO model. The acquisition functions for DKBO-AE, LA-MCTS, TuRBO and \roiTS are Thompson sampling on different GPs. The acquisition functions for \roiCI and \interCI are defined in \eqref{eq:acqROI-ci} and \eqref{eq:acqCI}, respectively. The average ranking corresponds to the ultimate simple regrets shown in \figref{fig:exp:simpleRegret}}. We highlight the rankings of the proposed methods.} \label{table:methods}

\end{table*}

\paragraph{Optimization Performance} The experiment results in \figref{fig:exp:simpleRegret} demonstrate the robust performance of \interCI which consistently matches or outperforms the best baseline. In contrast, LA-MCTS consistently matches or outperforms TuRBO-DK,  but lags behind DKBO-AE on the 1D Toy which indicates its potential inefficiency in the tasks of high-frequency areas hindering its partitioning of the search space. Note that we also find that using SVM to generalize the partition on $Y_t$ to $\searchSpace$ in LA-MCTS occasionally fails possibly due to the intrinsic complexity of the partition learned on $Y_t$ demanding methods of greater capability, while the level-set partition of \interCI is regularized by the smoothness of the global Gaussian process $\mathcal{GP}_{f_g}$. We reject the failed LA-MCTS trials.

{TuRBO-DK} matches \interCI performance on GB1 and but loses to DKBO-AE on all other cases.  By construction, 1D Toy and HDBO-200D could have a large amount of distant local maximum, while {TuRBO-DK} relies on the locality of the observation to identify the trust regions. {TuRBO-DK} could be potentially trapped in the local maximum and the performance degrades in the scenario where the multiple modules are distant from each other while the gap between sub-optimal and optimal observation is significant. In contrast, \interCI is capable of identifying multiple regions of interest with the level-set partitioning without specifying the desired number of regions.

On 1D-Toy dataset which is composed of the low-value high-frequency areas and the high-value low-frequency area, \interCI significantly outperforms the baselines and reaches the near-optimal area within 30 iterations. Due to the complexity of the low-value areas that make up a large portion of the objective, the GPs underfit especially at the earlys stage where access to observation is limited as is shown in \figref{fig:overview}. At this phase, the DKBO-AE stably outperforms \interCI potentially without the distraction from the under-fitting ROI GP. While on HDBO datasets which by construction bears relative uniform smoothness, the partition-based algorithms other than {TuRBO-DK} all enjoy similar benefits at the beginning stage.

On Nanophotonics, Water-Converter and Rosetta, \interCI matches or outperforms the baselines including \roiCI, while losing to \roiCI and \roiTS on GB1. This ablation study indicates the necessity of taking the intersection of CI in most scenarios, while revealing that more aggressive filtering of \roiCI could sometime be beneficial. \roiTS matches \algname on HDBO, Nanophotonics,  Water-Converter and Rosetta, reflecting that the {ROI} GP could be as informative as the combination of the global and {ROI} GP in some cases. The fact that \roiTS uniformly outperforms LA-MCTS and TuRBO-DK on all the experiments requires further study on integrating Thompson sampling into a \algname-style framework with a similar theoretical guarantee.

\rebuttal{We summarize the different methods tested in the experiments with \tabref{table:methods}. The comparison between \roiTS with both LA-MCTS and TuRBO shows the effectiveness of identifying the ROI on a point-basis on the given discretization with confidence intervals of a global GP. The comparison between \roiCI and \interCI highlights the benefits on integrating information from global model into optimization on ROI especially when using the confidence interval as acquisition function.}

% \vspace{-3mm}
\section{Conclusion}\label{sec:conclusion}

We propose a novel framework for adaptively learning regions of interest for Bayesian optimization. Our model maintains two Gaussian processes: One global model for identifying the ROIs as (adaptive) superlevel-sets; the other surrogate model for acquiring data in these high-confidence ROIs. We proposed to take the width of the intersection of the point-wise confidence intervals of both GPs as the acquisition function to achieve a theoretical guarantee on both the convergence rate of the filtering and optimization process. We demonstrate our algorithm in promising real-world experiment design scenarios, including protein engineering and material science. Our results show that {\algname} compares favorably against state-of-the-art BO approaches under similar settings---especially in high-dimensional and structured tasks with non-stationary dynamics---while having fewer hyperparameters to fine-tune. 

{More generally, we propose the principled framework combining the power of a coarse global model for filtering low-interest areas and a fine-grained local model for focused optimization, supported by theoretical insights.} We show the potential of the framework by integrating Thompson sampling, and the extensions to other acquisition functions that are not based on confidence intervals are also of interest for future work. We demonstrate the practical issues of taking the intersection of all historical CIs and discuss the cost of only taking the intersection of CIs at each time step. This raises the demand for future studies on addressing the dynamics of (deep) kernel learning.

% \vspace{-3mm}
\paragraph{Acknowledgement}
% \vspace{-2mm}
{This work was performed under the auspices of the U.S. Department of Energy by Lawrence Livermore National Laboratory under contract DE-AC52-07NA27344 and was supported by the LLNL-LDRD Program under Project No. 20-ERD-032. Lawrence Livermore National Security, LLC. LLNL-CONF-841006.
The DOD's Joint Program Executive Office for Chemical, Biological, Radiological and Nuclear Defense (JPEO-CBRND) under the GUIDE program, in collaboration with the Defense Health Agency (DHA) COVID funding initiative for Rapid co-design of manufacturable and efficacious antibody therapeutics for COVID-19 via a machine-learning-driven computational design platform, molecular dynamics simulations, and experimental validation, Lawrence Livermore National Laboratory (LLNL), Proposal L22260, Agreement ID\#44208 was used for this effort.
Fengxue Zhang was supported in part by NSF \#2037026.}

\bibliography{reference}

\clearpage
\begin{appendix}

\section{Proofs}
\subsection{Proof of \lemref{lem:CI} and \corref{cor:CI}}
\begin{proof}
% \TBD{Consider removing UCB.}
Similar to lemma 5.1 of \cite{srinivas2009gaussian}, with probability at least $1-\delta$, $\forall \instance\in \discreteSet, \forall t\geq 1, \forall f \in \{f_g, \hat{f}\}$,
$$\vert f(\instance) - \mu_{f, t-1}(\instance)\vert \leq \beta_{t}^{1/2}\sigma_{f, t-1}(\instance)$$

Note that we also take the union bound on $ f \in \{f_g, \hat{f}\}$.

Then  $\forall t\geq 1, \forall f \in \{f_g, \hat{f}\}$, 
$$P\paren{f^* \leq \UCBit_{f, t}(\instance^*) \leq \UCBit_{f, t, \max}} \geq 1-\delta$$
According to \eqref{eq:acqLCB+UCB}, $\forall t\geq 1$
$$P\paren{f^* \leq \max_{\instance \in \searchSpace} \acqUCB_t(x)} \geq 1-\delta$$

Symmetrically, $ \forall \instance\in \discreteSet, \forall t\geq 1, \forall f \in \{f_g, \hat{f}\}$,
$$P\paren{f^* \geq f(\instance) \geq \LCB_{f, t}(\instance)} \geq 1-\delta$$
Then $ \forall t\geq 1,$
$$P\paren{\UCBit_{f_g, t}(\instance) \geq f^* \geq \LCB_{f_g, t, \max}} \geq 1-\delta$$
according to the definition of $\roi$, $P\paren{\instance^* \in \roi_t} \geq 1-\delta$.

Also, according to \eqref{eq:acqLCB+UCB}, $\forall t\geq 1$
$$P\paren{f^* \geq \max_{\instance \in \searchSpace} \acqLCB_t(x)} \geq 1-\delta$$
\end{proof}

\subsection{Proof of \propref{prop:regret}}

The following two lemmas shows that the width of the interval is bounded by the maximum of \acq.

\begin{lem} \label{lem:acqCIBound}
Under the same conditions assumed in \lemref{lem:CI} except for $\beta_t=2\log(2\vert \discreteSet \cap \roi \vert \pi_t/ \delta)$, with acquisition function $\acq(\instance) = \vert \acqCI_{t}(\instance) \vert$, $\forall t\geq 1, \forall f \in \{f_g, \hat{f}\}$, let $\instance'' = \argmax_{\instance \in \discreteSet \cap \roi}\vert \acqCI_{t}(\instance) \vert$ we have $\max_{\instance \in \discreteSet \cap \roi}\acqUCB_{t}(\instance) - \max_{\instance \in \discreteSet \cap \roi}\acqLCB_{t}(\instance) \leq \rho_\CI\beta^{1/2}_{t}\sigma_{f,t-1}(\instance)$. Here $\rho_\CI \leq \rho_\UCBit \leq 2$.
\end{lem}

\begin{proof}
$\forall t\geq 1, \forall f \in \{f_g, \hat{f}\}$
\begin{align*}
    \max_{\instance \in \discreteSet \cap \roi}\acqUCB_{t}(\instance) - \max_{\instance \in \discreteSet \cap \roi}\acqLCB_{t}(\instance)
    & \leq \acqUCB_{t}(\instance')- \acqLCB_{t}(\instance')\\
    & \leq 2\beta_{t}^{1/2}\sigma_{f, t-1}(\instance')\\
    & \leq \argmax_{\instance \in \discreteSet \cap \roi}\vert \acqCI_{t}(\instance) \vert\\
    & = 2\beta_{t}^{1/2}\sigma_{f, t-1}(\instance'')
\end{align*}
\end{proof}

The followings finish the proof of \propref{prop:regret}.\\
\begin{proof}
By lemma 5.4 of \citet{srinivas2009gaussian}, with $\beta_t=2\log(2\vert \discreteSet \cap \roi \vert \pi_t/ \delta)$, $\forall f \in \{f_g, \hat{f}\}, \sum_{t=1}^{T} (2\beta_{t}^{1/2}\sigma_{f, t-1}(\instance_t))^2 \leq C_1\beta_T\maxInfo_{f, T}$. Taking the union bound of \lemref{lem:CI} and \corref{cor:CI2}, with probability at least $ 1-2\delta$, $\forall f \in \{f_g, \hat{f}\}$,
\begin{align*}
     \sum_{t=1}^{T} \paren{\max_{\instance \in \discreteSet \cap \roi}\acqUCB_{t}(\instance) - \max_{\instance \in \discreteSet \cap \roi}\acqLCB_{t}(\instance)}^2
    &\leq \sum_{t=1}^{T}(\rho_\acq \beta_{t}^{1/2}\sigma_{f, t-1}(\instance_t))^2\\
    &\leq \rho_\acq^2C_1\beta_T\maxInfo_{f, T}/4
\end{align*}
According to \eqref{eq:acqLCB+UCB}, $\max_{\instance \in \discreteSet \cap \roi}\acqUCB_{t}(\instance) - \max_{\instance \in \discreteSet \cap \roi}\acqLCB_{t}(\instance)$ is monotonically decreasing. By Cauchy-Schwaz, with probability at least $ 1-2\delta$, $\forall f \in \{f_g, \hat{f}\}$, \begin{align*}
    \rho_\acq^2C_1\beta_T\maxInfo_{f, T}/4
    &\geq \sum_{t=1}^{T} \paren{\max_{\instance \in \discreteSet \cap \roi}\acqUCB_{t}(\instance) - \max_{\instance \in \discreteSet \cap \roi}\acqLCB_{t}(\instance)}^2\\
    &\geq \frac{1}{T}(\sum_{t=1}^{T} \max_{\instance \in \discreteSet \cap \roi}\acqUCB_{t}(\instance) - \max_{\instance \in \discreteSet \cap \roi}\acqLCB_{t}(\instance))^2\\
    &\geq T \paren{\max_{\instance \in \discreteSet \cap \roi}\acqUCB_{T}(\instance) - \max_{\instance \in \discreteSet \cap \roi}\acqLCB_{T}(\instance)}^2
\end{align*}
Assume with probability at least $1-2\delta$,
\begin{align*}
    \paren{\max_{\instance \in \discreteSet \cap \roi}\acqUCB_{T}(\instance) - \max_{\instance \in \discreteSet \cap \roi}\acqLCB_{T}(\instance)}^2
    \leq \rho_\acq^2C_1\beta_T\maxInfoBALLET_{T}/4T 
    \leq \epsilon^2
\end{align*}
Hence, with the smallest $T$ satisfying $T \geq \frac{\rho^2_\acq \beta_T \maxInfoBALLET_T C_1}{4 \epsilon^2}$, $\Pr{\max_{\instance \in \discreteSet \cap \roi}\acqUCB_{T}(\instance) - \max_{\instance \in \discreteSet \cap \roi}\acqLCB_{T}(\instance)\leq \epsilon} \geq 1-2\delta$.
\end{proof}

\section{Discussions}
\paragraph{Smoothness improvement on ROI}

In near-optimal areas, the smoothness of the objective should be no worse than the smoothness in the larger (global) area. \cite{srinivas2009gaussian} discussed the role of smoothness in reducing $\maxInfo$. As indicated by \propref{prop:regret}, the benefits to optimization of a smoother kernel learned on ROI instead of the kernel learned on the globe could be reflected in the reduced $\maxInfoBALLET$ in the regret bound compared to $\maxInfo_{f_g}$ without the filtering of BALLET. 

%\clearpage
\section{Supplemental Experimental Results}

In this section, we include an extended empirical study of \algname, compared against a broader collection of baseline algorithms with varying hyperparameters. Specifically, we show the results for the following algorithms:
\begin{itemize}\denselist
    \item \textsl{BALLET-ICI-RBF}: \interCI with RBF (squared-exponential) base kernel $k(\instance, \instance') = \exp\left(- \frac{d(\instance, \instance')^2}{2l^2} \right)$.
    \item \textsl{BALLET-ICI-Lin}: \interCI with linear base kernel $k(\instance, \instance') = \sigma_0 ^ 2 + \instance \cdot \instance'$ (with prior $N(0, \sigma_0^2)$ on the bias).
    \item \textsl{DKBO-AE-RBF}. \DKBO with RBF base kernel \citep{zhang2022learning}.
    \item \textsl{DKBO-AE-Lin}. \DKBO with linear base kernel.
    \item \textsl{\LAMCTS}. The Latent Action Monte Carlo Tree Search algorithm (LA-MCTS) of \citet{wang2020learning}.
    \item \textsl{TuRBO-m}. The Trust region Bayesian optimization (\turbo) algorithm of \citet{eriksson2019scalable}, where $m$ specifies the variant of \turbo that maintains $m$ local models in parallel.
\end{itemize}

\begin{table}[h]
\centering
\scalebox{.82}{
    %\begin{tabular}{l|l|l|l|l|l|l}
    \begin{tabular}{l l l l l l l}
    \toprule
        ~ & \textbf{1-D toy} & \textbf{HDBO} & \textbf{Nanophotonics} & \textbf{WaterConverter} & \textbf{GB1} & \textbf{Rosetta} \\ 
        \midrule
        ~ & $T=40$ & $T=40$ & $T=90$ & $T=90$ & $T=90$ & $T=90$ \\ \hline
        \textsc{\algname-ICI-RBF} & \bm{$0.03\pm 0.01$} & \bm{$85.90\pm 7.29$} & $76.65\pm 9.55$ & $33664.62\pm 0.00$ & \bm{$4.81\pm 0.15$} & \bm{$3.86\pm 0.49$} \\ \hline
        \textsc{\algname-ICI-Lin} & $0.10\pm 0.03$ & $110.81\pm 7.93$ & {$34.49\pm 4.39$} & \bm{$20084.66\pm 2928.84$} & $6.33\pm 0.28$ & $5.11\pm 0.18$ \\ \hline
        \textsc{\DKBO-RBF}  & $0.05\pm 0.02$ & $90.75 \pm 16.01$ & $89.49\pm 3.44$ & $28591.49\pm 2560.23$ & $5.02\pm 0.41$ & $4.89\pm 0.16$ \\ \hline
        \textsc{\DKBO-Lin} & $0.07\pm 0.01$ & $92.84\pm 6.22$ & $82.94\pm 4.50$ & $33664.63 \pm 0$ & $6.44\pm 0.19$ & $4.12\pm 0.46$ \\ \hline
        \LAMCTS & $0.10\pm 0.04$ & $95.47\pm 4.84$ & \bm{$30.79\pm 10.28$} & $26814.43\pm 1593.76$ & $5.59\pm 0.40$ & $5.09\pm 0.32$ \\ \hline
        \textsc{\turbo-1} & $0.31\pm 0.00$ & $136.80\pm 0.00$ & $96.58\pm 0.00$ & $33664.69\pm 0.00$ & $5.34\pm 0.52$ & $6.67\pm 0.00$ \\ \hline
        \textsc{\turbo-2} & $0.07\pm 0.06$ & $105.51\pm 5.47$ & $50.60\pm 15.46$ & $28450.65\pm 1691.96$ & $4.95\pm 0.45$ & $5.75\pm 0.32$ \\ \hline
        \textsc{\turbo-4} & $0.04\pm 0.03$ & $93.74\pm 10.25$ & $63.60\pm 3.48$ & $32800.93\pm 3148.98$ & $5.72\pm 0.72$ & $5.46\pm 0.22$\\ 
        \bottomrule
    \end{tabular}
    }
    \caption{
    Simple regret (Mean $\pm$ SE) at the $T^{\text{th}}$ iteration on the 6 datasets described in \secref{sec:exp}. Here, $T$ aligns with the optimization horizon reported in \figref{fig:exp:simpleRegret} for each dataset. The top results are highlighted in bold.}\label{tab:res:parameter}
    \vspace{-5mm}
\end{table}

As shown in \tabref{tab:res:parameter}, \interCI (with different choices of base kernels) consistently outperforms other baselines on all datasets but Nanophotonics. On Nanophotonics, while there is a small gap between \interCI and \LAMCTS, \LAMCTS is relatively unstable with a larger variance (SE). This is consistent with the results reported in \figref{fig:exp:simpleRegret}.

\paragraph{Hyperpameter choice}

We further provide results on \algname's performance with varying $\beta$ when filtering. \Figref{fig:beta-vs-regret} shows the simple regret of \interCI on the \textit{Nanophotonics} dataset. We observe that---although our regret bounds in \secref{sec:interci} rely on specific choices of $\beta_t^{1/2}$ for filtering -- the empirical results are robust within a range of small values. Also, using the $\beta_T^{1/2}=6.2$ as the analytic results in \propref{prop:regret} failed to match the performance of the fixed $\beta_t^{1/2}\leq 1$, showing its over-conservative problem.

\begin{figure}[th]
  \begin{center}
      \includegraphics[trim={0pt 10pt 0pt 0pt}, width=\includegraphics[trim={0pt 10pt 0pt 0pt}, width=.45\textwidth]{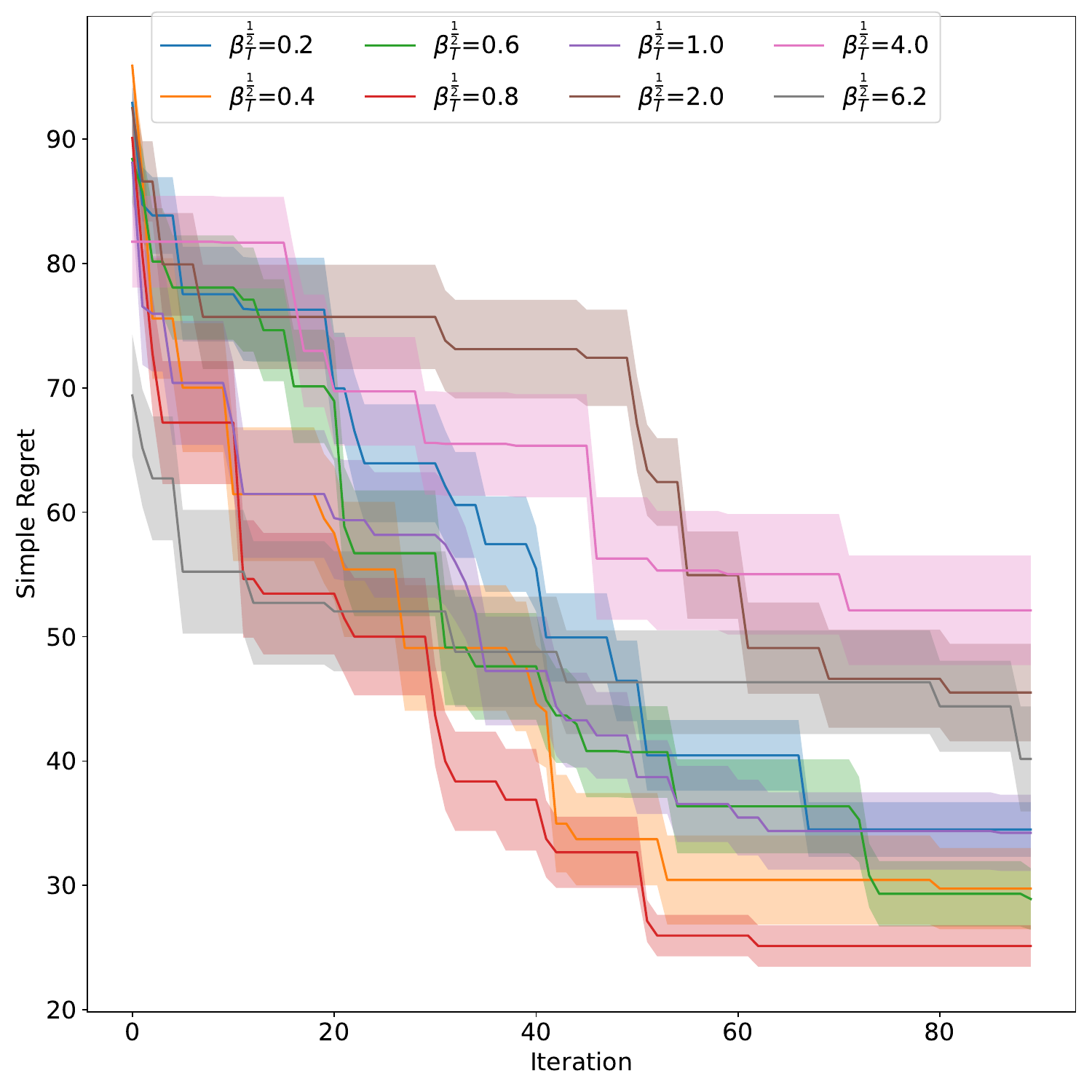}
      \caption{Effect of scaling parameter $\beta$ on the \textit{Nanophotonics} dataset. For $\beta_T^{1/2} \leq 4$, the values are fixed through the optimization, while $\beta_T^{1/2} =6.2$ corresponds to the results of varying $\beta_t^{1/2}$ as in \propref{prop:regret}.}
      \label{fig:beta-vs-regret}
      \end{center}
\end{figure}
%\end{wrapfigure}

\section{Additional Results}
\subsection{TRUVAR Results}
We do not include TRUVAR by \citep{bogunovic2016truncated} in the main paper for the following reasons. (1) TRUVAR is \textbf{not a partition-based BO method} that aims at resolving the heteroscedasticity in BO by learning local models;
(2) It is prohibitive to run for large candidate sets as TRUVAR's acquisition function requires estimating the posterior variance reduction for all the remaining candidates. 
We observe on the 1D-toy dataset the simple regret is $0.121\pm 0.033$ by TRUVAR v.s. $0.0031 \pm 0.011$ by BALLET-ICI.

\vspace{-3mm}
\subsection{Exact-GP results}
\todo{add new graphs}
We Compare Exact-GP results on 1-D Toy, Nanophotonics, and Water converter configuration datasets with DKBO-AE as an ablation study of deep kernel learning. The choice of kernels and hypereparameters are identical to the deep kernel discussed in \secref{sec:exp} except for removing the latent space mapping and kernel interpolation. As is shown in \figref{fig:exp_append_RCI}, Exact-GP is consistently outperformed by DKBO-AE and \interCI.

% \vspace{-3mm}
\subsection{RCI results}
We compare DKBO-AE-RCI directly with DKBO-AE as the direct ablation study of the proposed acquisition function. The choice of kernels and hypereparameters are identical to the deep kernel discussed in \secref{sec:exp}. The acquisition function 
$ \UCBit_{\globalf, t}(\instance) - \LCB_{\globalf, t}$, which is similar to \eqref{eq:acqROI-ci}, is maximized over $\searchSpace$ instead of $\roi$. 

\begin{figure*}[h] 
        \centering
        \begin{subfigure}[th]{1\linewidth}
        {
          \centering
          \hspace{12mm}
          \includegraphics[trim={0pt 0pt 0pt 0pt}, width=.85\textwidth]{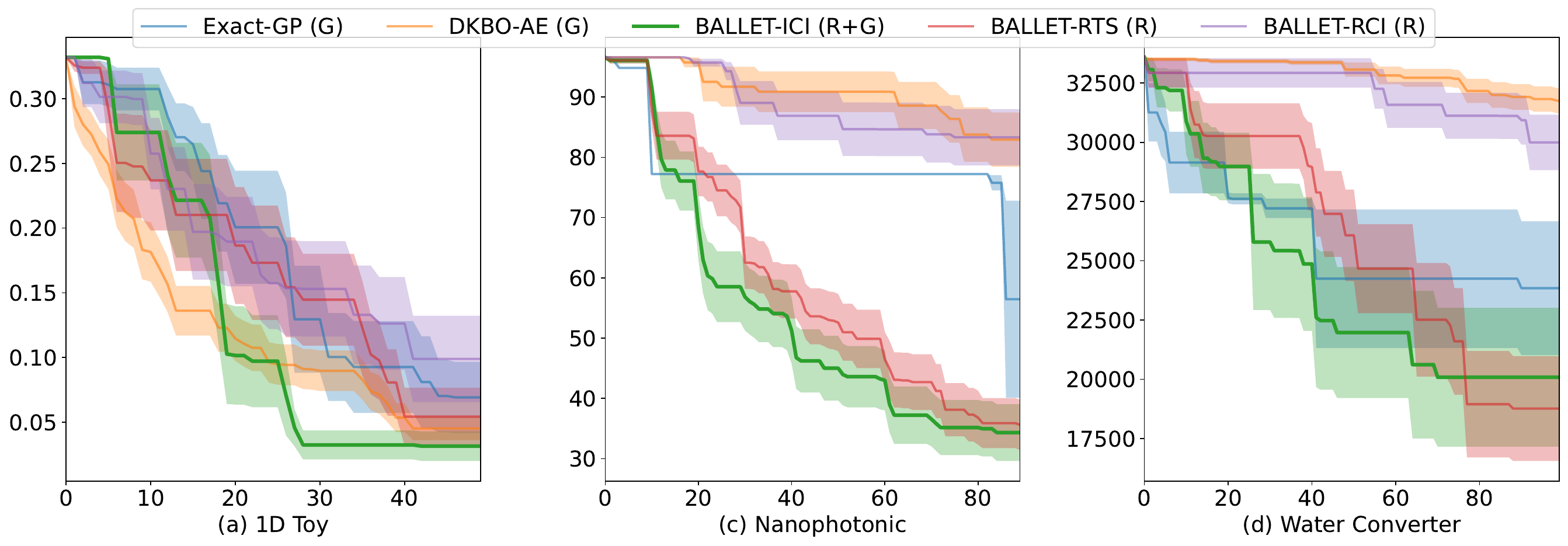}
        }
        \end{subfigure}
        \begin{subfigure}[bh]{1\linewidth}
        {
          \centering
          \hspace{10mm}
          \includegraphics[trim={0pt 0pt 0pt 0pt}, width=.865\textwidth]{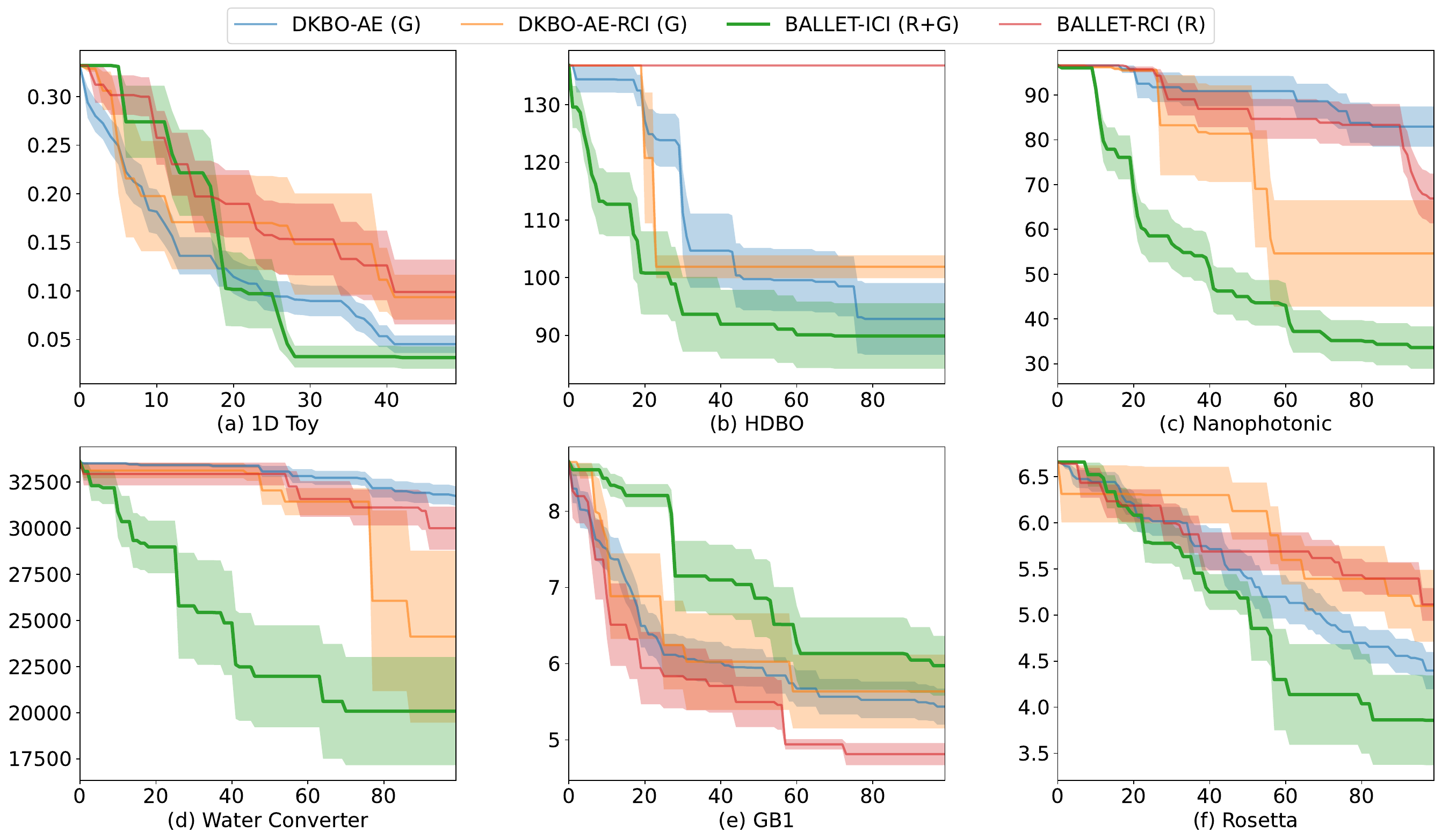}
        }
        \end{subfigure}
  \caption{\small Simulation results on each task are shown here. The error bar demonstrates the standard error. The x-axis denotes the number of iterations, and the y-axis denotes the simple regret. The simple regrets for the 10 initial randomly picked warm-up datasets are clipped. (G), (R), and (R+G) means the global model only, the ROI model only, and the ROI model combined with the global model correspondingly.
  }    \label{fig:exp_append_RCI}
\end{figure*}

As is shown \figref{fig:exp_append_RCI}, \interCI outperforms the baselines except on GB1, indicating the advantage of leveraging both global and local information together. \roiCI performs the best and DKBO-AE-RCI outperforms \interCI on GB1. This shows the benefits of identifying the ROI and optimizes on it, and the harm a potential discrepancy between the global model and the ROI model could be to the optimization.

\subsection{Larger Horizon for Protein Datasets}
\rebuttal{We provide additional large-scale empirical results on the Rosetta-86D and GB1-119D pre-collected protein design datasets. The results are collected from 10 independent 300-iteration trials for each experiment. 
\begin{table}[h]
\centering
\begin{tabular}{lcccc}
\hline
& DK-AE & BALLET-ROI & BALLET-Intersection & Avg Run Time (sec) \\
\hline
TS       & 4.58$\pm$0.16 (1.15e-03) & \textbf{3.73$\pm$0.17} & nan & 1010.7 \\
UCB      & 4.78$\pm$0.22 (1.05e-03) & 4.60$\pm$0.33 (1.89e-02) & 4.44$\pm$0.31 (3.50e-02) & 723.2 \\
EI       & 5.02$\pm$0.34 (2.11e-03) & 4.48$\pm$0.13 (1.45e-03) & nan & 789.2 \\
CI       & 4.61$\pm$0.16 (9.98e-04) & 4.72$\pm$0.35 (1.20e-02) & 4.21$\pm$0.40 (1.51e-01) & 667.7 \\
LA-MCTS  & 5.60$\pm$0.32 (5.84e-05) & nan & nan & 2521.7 \\
TuRBO-DK & 6.30$\pm$0.50 (9.64e-05) & nan & nan & 1205.4 \\
\hline
\end{tabular}
\caption{Performance comparison on GB1.}\label{tab:gb1_performance}
\end{table}

\begin{table}[h]
\centering
\begin{tabular}{lcccc}
\hline
& DK-AE & BALLET-ROI & BALLET-Intersection & Avg Run Time (sec) \\
\hline
TS       & 3.95$\pm$0.18 (3.84e-04) & 3.51$\pm$0.40 (1.97e-02) & nan & 126.1 \\
UCB      & 3.55$\pm$0.30 (8.25e-03) & 2.43$\pm$0.42 (3.99e-01) & 2.43$\pm$0.42 (3.99e-01) & 93.9 \\
EI       & 3.38$\pm$0.38 (2.90e-02) & 3.01$\pm$0.39 (1.01e-01) & nan & 91.1 \\
CI       & 3.56$\pm$0.38 (1.43e-02) & 4.31$\pm$0.35 (5.38e-04) & \textbf{2.28$\pm$0.35} & 82.5 \\
LA-MCTS  & 3.02$\pm$0.46 (1.19e-01) & nan & nan & 1028.1 \\
TuRBO-DK & 3.80$\pm$0.32 (3.32e-03) & nan & nan & 433.1 \\
\hline
\end{tabular}
\caption{Performance comparison on Rosetta.}\label{tab:rosetta_performance}
\end{table}

We summarize the simple regrets of these two additional experiments in \tabref{tab:gb1_performance} and \tabref{tab:rosetta_performance}. We highlight the best results and report the p-value against all other methods in the parenthes. The results are shown in three columns: (1) using deep kernel only, (2) applying the acquisition function within the identified ROI, and (3) using the intersection of confidence intervals from both global and ROI models as the acquisition function. The average running times for a single trial of the right-most (fastest within the row) results are provided for each row.

\begin{figure*}[h!] 
        \centering
        \begin{subfigure}[th]{0.95\linewidth}
        {
          % \centering
          % \hspace{12mm}
          \includegraphics[trim={0pt 0pt 0pt 0pt}, width=.95\textwidth]{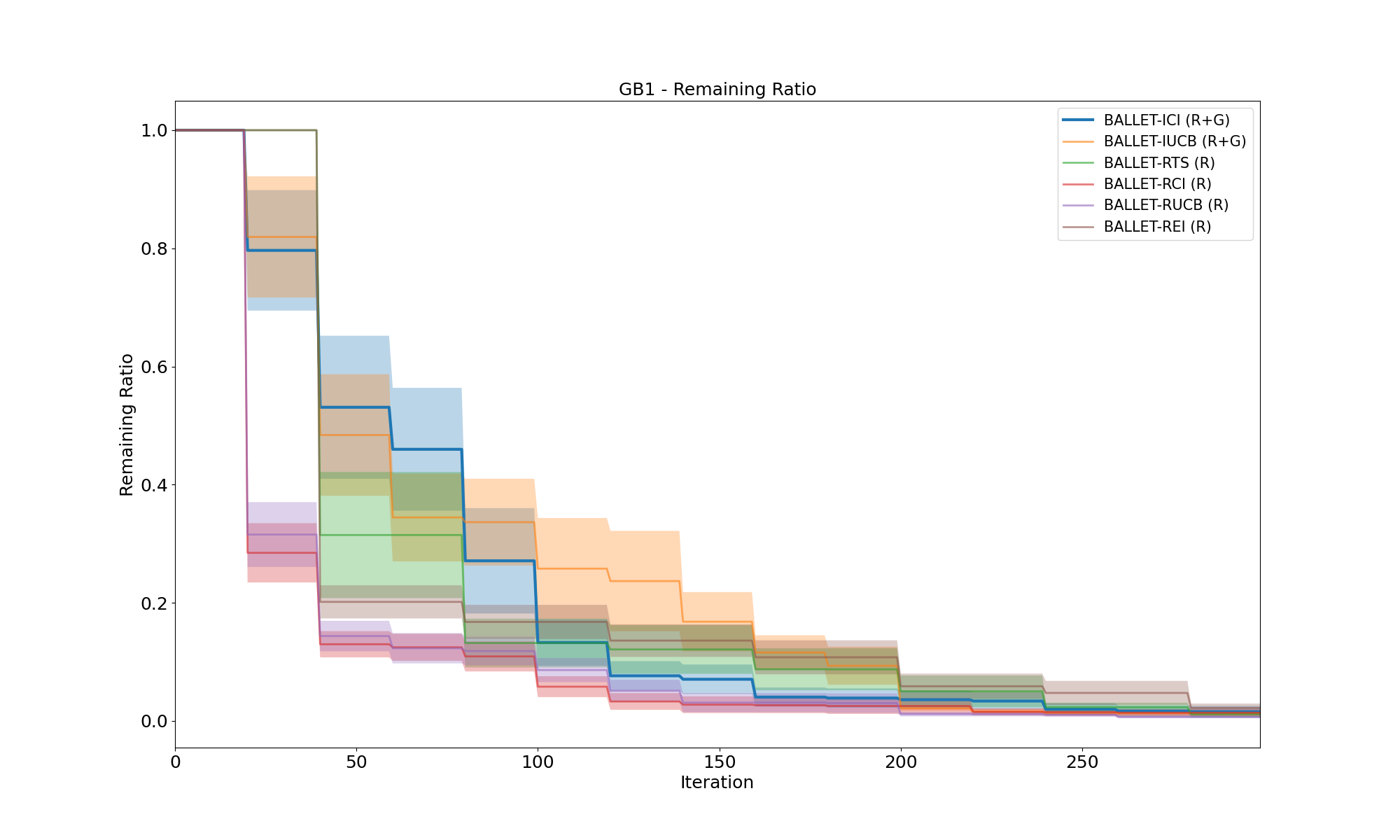}
        }
        \end{subfigure}
        \begin{subfigure}[bh]{.95\linewidth}
        {
          % \centering
          % \hspace{10mm}
          \includegraphics[trim={0pt 0pt 0pt 0pt}, width=.95\textwidth]{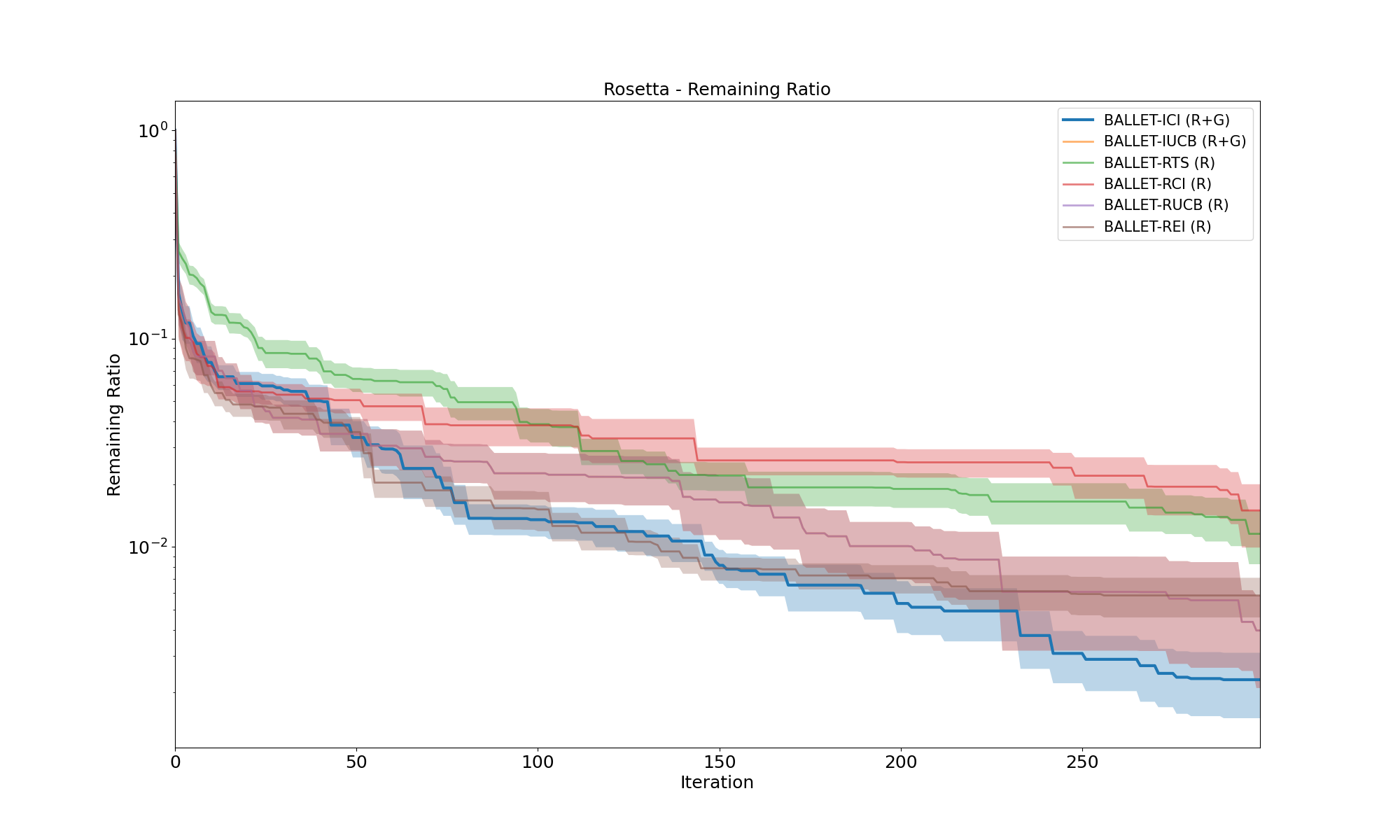}
        }
        \end{subfigure}
    \vspace{-2mm}
  \caption{\small Simulation results on each task are shown here. The error bar demonstrates the standard error. The x-axis denotes the number of iterations, and the y-axis denotes the filtering ratio. The ratio for the 10 initial randomly picked warm-up datasets are clipped. (G), (R), and (R+G) means the global model only, the ROI model only, and the ROI model combined with the global model correspondingly.
  }   \label{fig: additional_ratio}
\end{figure*}

The results demonstrate that using the \algname framework with each acquisition function leads to improvements over the DKBO-AE version. Additionally, comparing BALLET-Intersection and BALLET-ROI reveals the benefits of using the intersection of confidence intervals. The filtering ratio curve shown in \figref{fig: additional_ratio}, combined with these results, suggests that the proposed \interCI method accelerates the optimization process by shrinking the ROI more rapidly.}
% \newpage

\section{Discussions}

\subsection{Computational Cost}
In deep kernel learning, which is shown to bear strong empirical performance in regression and optimization 
(e.g.\citep{pmlr-v51-wilson16, wistuba2021few}), 
the learning cost is $\mathcal{O}(n)$ for $n$ training points, and the prediction cost is
$\mathcal{O}(1)$ 
per test point and is more efficient than the exact GP in terms of computational cost. Compared to the \textbf{significant experiment cost} in the real-world application BALLET is proposed for (e.g., cosmological design, protein study), the computational cost is negligible. % and hence not a major concern in BALLET. 
Meanwhile, the runtime of other partition-based algorithms depends on the 
\textbf{hyperparameters} of the partitioning heuristics,
e.g., K-means iterations in LA-MCTS, the number and size of trust regions in TuRBO.

\subsection{Limitation and Future Work}
We summarize the following limitations throughout the paper.
\begin{itemize}
    \item The analysis only applies to given discretization, while sampling and related work focus on the issue;
    \item It Doesn’t help to learn an ROI GP when the objective has global uniformity. The global kernel itself forms a good surrogate GP.
    \item The analysis should be able to extend to more acquisition functions.
    \item Lack of analysis on top of the (deep) kernel learning. Though different from applying an exact GP through the optimization process, deep kernel learning has shown strong performance in regression and optimization tasks. The gap between DK-based BO and exact GP-based BO remains to be filled.
\end{itemize}

\end{appendix}

\end{document}